%% file: root.tex
\documentclass[twoside]{article}
\usepackage{balance}
\usepackage[accepted]{aistats_copy_arxiv}
\usepackage{xcolor}
\usepackage[colorlinks = true,
            linkcolor = blue,
            urlcolor  = blue,
            citecolor = blue,
            anchorcolor = blue]{hyperref}
\usepackage{multirow}
\usepackage{amsfonts,amsthm,amsmath,mathtools,bbm}       
\usepackage{nicefrac,xcomment}       
\usepackage{microtype}      
\usepackage{xcolor}         
\usepackage{xspace}
\usepackage{verbatim}
\usepackage{wrapfig}
\usepackage[pdftex]{graphicx}
\usepackage[noend,ruled,vlined,linesnumbered]{algorithm2e}
\usepackage[para]{footmisc}
\usepackage{subfig}
\usepackage{array}
\usepackage{stmaryrd}
\newcolumntype{P}[1]{>{\centering\arraybackslash}p{#1}}
\newcolumntype{M}[1]{>{\centering\arraybackslash}m{#1}}

\newenvironment{CompactItemize}{
\begin{list}{$\bullet$}{%
\setlength{\leftmargin}{12pt}
\setlength{\itemindent}{1pt}
\setlength{\topsep}{1pt}
\setlength{\itemsep}{-2pt}
}}
{\end{list}}

\newcommand{\parfrac}[2]{\paran{\frac{#1}{#2}}}
\newcommand{\paran}[1]{\left( #1 \right)}
\newcommand{\pnorm}[2]{\left|\left| {#2} \right|\right|_{#1}}

\newcommand{\abs}[1]{\left| #1 \right|}

\newcommand{\ZZZ}{\mathbb{Z}}
\newcommand{\zero}{$0^+$}
\newcommand{\logloss}{\textsc{CEloss}\xspace}
\newcommand{\Klogloss}{\textsc{CEloss}\xspace}
\newcommand{\LabelInf}{\textsc{LabelInf}\xspace}

\usepackage{tikz}
\usetikzlibrary{arrows}
\usetikzlibrary{decorations.pathreplacing}

\newtheorem{definition}{Definition}
\newtheorem{lemma}{Lemma}
\newtheorem{theorem}{Theorem}
\newtheorem{proposition}{Proposition}
\newtheorem{corollary}{Corollary}
\newtheorem{remark}[theorem]{Remark}

\usepackage[square]{natbib}

\begin{document}

\twocolumn[

\aistatstitle{Reconstructing Test Labels from Noisy Loss Functions}

\aistatsauthor{Abhinav Aggarwal, Shiva Prasad Kasiviswanathan, Zekun Xu, \\\textbf{Oluwaseyi Feyisetan, Nathanael Teissier}}

\aistatsaddress{Amazon, USA\\\texttt{\{aggabhin,kasivisw,zeku,sey,natteis\}@amazon.com}} ]

\begin{abstract}
\vspace{-1em}
Machine learning classifiers rely on loss functions for performance evaluation, often on a private (hidden) dataset. In a recent line of research, label inference was introduced as the problem of reconstructing the ground truth labels of this private dataset from just the (possibly perturbed) cross-entropy loss function values evaluated at chosen prediction vectors (without any other access to the hidden dataset). In this paper, we formally study the necessary and sufficient conditions under which label inference is possible from \emph{any} (noisy) loss function value. Using tools from analytical number theory, we show that a broad class of commonly used loss functions, including general Bregman divergence-based losses and multiclass cross-entropy with common activation functions like sigmoid and softmax, it is possible to design label inference attacks that succeed even for arbitrary noise levels and using only a single query from the adversary. We formally study the computational complexity of label inference and show that while in general, designing adversarial prediction vectors for these attacks is co-NP-hard, once we have these vectors, the attacks can also be carried out through a lightweight augmentation to any neural network model, making them look benign and hard to detect. The observations in this paper provide a deeper understanding of the vulnerabilities inherent in modern machine learning and could be used for designing future trustworthy ML. 
\end{abstract}

\section{Introduction}
\vspace{-0.5em}
\input{introduction}
\vspace{-0.5em}
\section{Reducing Robust Label Inference to Codomain Separability} \label{sec:prelim}
\input{codomain_intro}

\vspace{-0.5em}
\section{Linear-Decomposability and Sub- Exponential Time Label Inference}
\label{sec:linearly_separable}
\input{linearly_separable}
\vspace{-0.5em}
\section{Multiclass Cross-Entropy Loss} \label{sec:logloss}
\vspace{-0.5em}
\input{log_loss_analysis}
\vspace{-0.75em}
\section{One Model to Infer Them All}
\label{sec:neural_network}
\vspace{-0.75em}
\input{neural_network}

\vspace{-1em}
\section{Empirical Analysis}
\label{sec:experiments}
\vspace{-0.5em}
\input{empirical_analysis}

\noindent\textbf{Concluding Remarks.}
\label{sec:conclusion}
\input{conclusion}

\bibliography{ref}
\appendix
\input{supplement}
\input{supplement_experiments}

\end{document}

%% file: introduction.tex
Consider a situation where a machine learning (ML) modeler is interacting with a data curator who owns a private dataset for a classification task. The curator agrees to evaluate on this private dataset the prediction vector (or an ML model) that the modeler submits, and replies back with loss function values. Such a situation is commonly encountered in machine learning competition settings like Kaggle~\citep{kaggle}, KDDCup~\citep{kdd_cup}, and ILSVRC Challenges~\citep{ilsvrc}. In some competitions, the features of the private (hold-out) dataset are revealed but not its labels, and the modeler submits the prediction vector on those features. In some other competitions, no information about the private dataset is revealed (i.e., neither the features nor the labels). The modeler submits a model that is then evaluated on the private dataset. 
A similar situation also appears when dealing with sensitive datasets, where either labels, or, both features and labels could be considered sensitive, and a modeler and curator interact through loss scores.

In this paper, we investigate if it is possible for a (malicious) modeler to recover \emph{all} the private labels using these interactions with the data curator (server). More broadly, we investigate the problem of robust label inference, where the goal is to infer the labels of a hidden dataset from only the (noisy) loss function queries evaluated on the dataset. Of particular interest will be the case where the modeler gets {\em just one} loss query output, which could be distorted by noise. Surprisingly, we show that even with just this single query (and no access to the private feature set or any side knowledge), for many common loss functions including general Bregman divergence-based losses and multiclass cross-entropy with common activation functions like sigmoid and softmax, our inference attack succeeds in exactly recovering {\em  all} the labels. This is a stronger privacy violation than that postulated by {\em blatant non-privacy}~\citep{dinur2003revealing}, where the goal is to only reconstruct a good fraction of the true labels. 

Our observations in this paper have important ramifications, for example, when used by an adversary to execute a privacy breach by learning labels associated with a sensitive dataset, or by an unscrupulous participant to an ML competition for learning the unknown test labels. Our results call to attention these vulnerabilities which might be currently under silent exploitation. Armed with this information, individuals and organizations, which vend these seemingly innocuous aggregate metrics from their models can grasp the potential scope of the resulting information leakage.

\noindent\textbf{Overview of Our Results and Techniques.} Our attacks rely on a mathematical notion of {\em codomain separability} of loss functions, which posits that the output of the loss function is \emph{sufficiently} distinct on every possible labeling of the input datapoints (see Definition~\ref{def:codomain_separability}). We assume that the curator that returns the loss scores can add noise to these scores up to some (known) error bound $\tau$. This noise can also be introduced as error when the scores are communicated over noisy channels or computed on low-precision machines\footnote{Our assumptions about the noise generation process ensure that our attacks succeed irrespective of the noise process used by the data curator. Knowledge about the noise distribution can be helpful though. For random noise, by first generating a bound on the noise using a tail bound, our techniques can be applied.}. As one would expect, separating the loss function outputs by more than $2\tau$ is a necessary and sufficient condition for this label recovery to be accurate (see Proposition~\ref{prop:bound_on_tau_for_label_inference}). While intuitive, this result gives a natural candidate for label inference, from just one loss query, using an exhaustive (exponential) local search (see~\LabelInf~\eqref{eq:labelinf}). 


Throughout this paper, we assume that the adversary knows the loss function, number of datapoints $N$ and an upper bound $\tau$ on the resulting error (noise). We also assume that the loss is computed on all the datapoints. The main technical challenge here is to design prediction vectors for which a loss function demonstrates the required codomain separability to handle arbitrary noise levels. Our key idea here is to use sets with distinct subset sums. Two simple examples of such sets of size $n$ are $\{1,2,4,\dots,2^{n-1}\}$ and $\{\ln p_1, \dots, \ln p_n\}$, where $p_1,\dots,p_n$ are distinct primes. Sets like these are useful when characterizing the sufficient conditions under which the required codomain separability can be achieved. For example, in the binary classification setting with $N$ datapoints, the following problem comes up often in our analysis: $\text{Construct $\theta = [\theta_1,\dots,\theta_N]$ such that}$ 
\begin{align*}
\ \min_{\sigma_1,\sigma_2 \in \{0,1\}^N}\abs{\sum_{i:\sigma_1(i)=1}g(\theta_i) - \sum_{j:\sigma_2(j)=1}g(\theta_j)} \geq b,    
\end{align*}
for some function $g$ and bound $b$. To satisfy this inequality, it suffices to set $\theta$ such that the set $g(\theta) := \{g(\theta_i),\dots,g(\theta_N)\}$ has all distinct subset sums. This is because the summation operators essentially filter out subsets of elements from the vector $\theta$, and because $\sigma_1 \neq \sigma_2$ in the minima operator, these subsets must differ in at least one element. Now, to ensure that the minimum difference of the subset sums in $g(\theta)$ is at least $b$, one can solve for $g(\theta_i) = 2^ib$ or $g(\theta_i) = b\ln p_i$ (or using some other set with distinct subset sums) depending on actual form of $g$ and the application.

We use these ideas and tools from analytical number theory to provide constructions of adversarial prediction vectors for broad classes of ML loss functions based on Bregman divergences (Section~\ref{sec:linearly_separable}) and multiclass cross-entropy (Section~\ref{sec:logloss}), for both the unnoised and the noised setting. The analytical properties of squarefree integers also helps us to reduce the computation time needed by the adversary.
In addition to the single query model where the adversary has to work with only one (noisy) loss function value, we also analyze extensions where the adversary has access to multiple (noisy) loss function values from different prediction vectors. This extension comes in handy as with sufficient queries the local computation time required at the adversary becomes polynomial. Additionally, to handle situations an actual ML model is required (and not just the prediction vector), we provide a construction of a feed-forward neural network, which can be used to carry out these label inference attacks while making them look benign (see Section~\ref{sec:neural_network}). We also point out some caveats associated with our approaches on machines with finite floating point precision.




\noindent\textbf{Defenses.} Our focus in this paper is on characterizing the vulnerability of loss functions in leaking private information. Viewed from this perspective, our results establishes lower bounds on the amount of noise needed (as a function of precision, number of queries, etc) on releasing these loss functions for any {\em reasonable notion} of label privacy. A rigorous defense mechanism against our proposed attack would be to release the loss scores under on differential privacy~\citep{dwork2006calibrating} with carefully calibrated noise that overcomes this lower bound. This will ascertain desired levels of plausible deniability on the labels recovered by an adversary.

We also highlight that in general, determining whether a loss function is codomain separable is co-NP-hard (Theorem~\ref{thm:conp}). We establish this through a (polynomial time) Karp reduction from the {\em Almost Tautology} problem from Boolean satisfiability theory (see Appendix~\ref{app:hardness_results}). Based on standard consensus on the complexity of this class of problems, it is unlikely that there is a general polynomial time algorithm for robust label inference from loss functions~\citep{arora2009computational}.


\noindent\textbf{Related Work.} Label inference attacks were first introduced in~\citep{whitehill2018climbing} for binary log-loss using a heuristic solution to a min-max optimization problem. This attack does not recover all the labels and works only in the unnoised setting. The noised setting for binary log-loss was recently studied by~\citep{aggarwal2021icml}. While their approach was not formalized using codomain separability, their construction also used the idea of making the loss function outputs distinct for each labeling of the dataset using distinct subset sums. However, their algorithm runs in exponential time and works only in the unnoised setting for the multiclass case. Furthermore, they also restrict only to cross-entropy loss. Our paper not only subsumes these results, but also significantly extends them by showing that most commonly used loss functions in ML applications are vulnerable to leaking private information about the ground truth labels. Moreover, we provide single query sub-exponential time and multi-query polynomial time attacks that can be carried out through benign looking ML models and settle the computational complexity of robust label inference from arbitrary loss functions.\footnote{
Our approach is also reminiscent of similar concepts used in information theory, e.g., coding schemes based on Sidon sequences~\citep{o2004complete} and Golomb rulers~\citep{robinson1967class}, where the goal is to have a high minimum distance between the codewords.}.







%% file: codomain_intro.tex
We begin our discussion by formally defining the notion of codomain separability and its connections to label inference in the noised as well as unnoised setting. Our objects of interest are functions whose domain is the Cartesian product of the space of all labelings (defined by the  $\ZZZ_K^N = \ZZZ_K \times \dots \times \ZZZ_K \ \text{($N$ times)} = \{0,\dots,K-1\}^N$) and an arbitrary set $\Theta \subseteq \mathbb{R}^N$. Here, $K \geq 2$ represents the number of label classes.
This formulation captures the common scenario in machine learning, where we evaluate a loss function using the true labeling in $\ZZZ_K^N$ for $N$ datapoints based on a (prediction) vector in $\mathbb{R}^N$ generated by an ML model. $\Theta$ is the space of prediction vectors, and for $\theta =[\theta_1,\dots,\theta_N] \in \Theta$, the value of $\theta_i$ encodes the label prediction for the $i$th datapoint. We work with different loss functions that place different restrictions on $\Theta$. All missing details from this section are presented in Appendix~\ref{app:prelim}.

\noindent\textbf{Codomain Separability.} Informally, we call a function codomain separable if there exists some vector $\theta \in \Theta$ such that the function output is distinct on each $\ZZZ_K^N$ (keeping $\theta$ fixed). Thus, when $\theta$ is known, this one-one correspondence between the function's output and the labelings in $\ZZZ_K^N$ can be exploited to exactly recover all the labels from just observing the output.
\begin{definition} [$\tau$-codomain separability]
\label{def:codomain_separability}
Let $f: \ZZZ_K^N \times \Theta \to \mathbb{R}$ be a function. For $\theta \in \Theta$, define $\Lambda_{\theta}(f) := \min_{\sigma_1,\sigma_2 \in \ZZZ_K^N} \abs{f(\sigma_1,\theta) - f(\sigma_2,\theta)}$ to be the minimum difference in the function output keeping $\theta$ fixed. For a fixed $\tau > 0$, we say that $f$ admits \emph{$\tau$-codomain separability} using $\theta$ if $\Lambda_{\theta}(f) \geq \tau$. In particular, we say that $f$ admits \emph{\zero-codomain separability} using $\theta$ if there exists any $\tau > 0$ such that $\Lambda_{\theta}(f) \geq \tau$. 
\end{definition}
Compared to $\tau$-codomain separability for $\tau >0$, the \zero-codomain separability is weaker as it only requires $\Lambda_{\theta}(f) > 0$. This condition is used for label inference in the unnoised case. As an example for $\tau$-codomain separability, consider the function 
$f(\sigma,\theta) = \langle \sigma, \theta \rangle$ for $\sigma \in \{0,1\}^N$. To demonstrate codomain separability, it suffices to set $\theta = [1,2,4,\dots,2^{N-1}]$
which makes $f$ admit $1$-codomain separability. Multiplying each entry in $\theta$ by $\tau$ will make $f$ admit $\tau$-codomain separability for any $\tau > 0$. In upcoming sections, we discuss our constructions that make many popular loss functions separable. In Appendix~\ref{sec:negative_results}, we present some function classes that are provably not $\tau$-codomain separable.  



 

\noindent\textbf{Robust Label Inference.} The goal with robust label inference is to recover the true labeling (in $\ZZZ_K^N$) upon observing only the loss function output, even if noised. Observe that the results trivially hold for the unnoised case if we can handle arbitrary noise levels. 
\begin{definition}[$\tau$-Robust Label Inference]
\label{def:robust_label_inference}
Let $f:\ZZZ_K^N \times \Theta \to \mathbb{R}$ be a function, and $\sigma^\star \in \ZZZ_K^N$ be the (unknown) true labeling. For a given $\tau > 0$, we say that $f$ admits $\tau$-robust label inference if there exists $\theta \in \Theta$ and an algorithm (Turing machine) $\mathcal{A}$ that can recover $\sigma^\star$ given any $\ell \in \mathbb{R}$ where $\abs{f(\sigma^\star,\theta) - \ell} < \tau$, \emph{i.e.,} for all $\sigma^\star \in \ZZZ_K^N$, we have $\mathcal{A}\paran{\theta, N, \ell} = \sigma^\star$.
\end{definition}
We reiterate an important point to note here that $\tau$-robust label inference requires perfect reconstruction of $\sigma^\star$, which is a {\em stronger} notion than that required by notions like {\em blatant non-privacy}~\citep{dinur2003revealing}, where the goal is to only reconstruct a good fraction of $\sigma^\star$. Also, while the above definition is based on a single query, we later relax this requirement to study robust label inference under a multi-query model.

The following proposition formally establishes the connection between the above definitions of codomain separability and robust label inference. 

\begin{proposition}\label{prop:bound_on_tau_for_label_inference}
For any $\tau>0$, the function $f$ admits $\tau$-robust label inference using $\theta \in \Theta$ iff $\Lambda_{\theta}(f) \geq 2\tau$. 
\end{proposition}

Suppose the adversary picks $\theta \in \Theta$ based on Definition~\ref{def:codomain_separability} and gets back the noisy loss function value $\ell$ from the curator (server). Proposition~\ref{prop:bound_on_tau_for_label_inference} then allows for a natural label inference algorithm (which we call~\LabelInf) which iterates over all possible labelings to recover the one which is closest to the observed loss score:
\begin{align} \label{eq:labelinf}
    \mathbf{\LabelInf:}\;\; \sigma^* \gets \arg\min_{\sigma\in \ZZZ_K^N} \abs{f(\sigma,\theta) -\ell} 
\end{align}
A special case of \LabelInf is the unnoised setting, wherein $\ell = f(\sigma^\star,\theta)$. In that case, it suffices to design a vector  $\theta \in \Theta$ with respect to which $f$ admits \zero-codomain separability and $\tau$ plays no role.

While intuitive, an important feature about the approach outlined in \LabelInf~\eqref{eq:labelinf} is that it makes just one call to the server to retrieve the (loss) function $f$ evaluated at a single $\theta$, but still reconstructs the entire private vector. However, the exponential time exhaustive search over the space of all labelings makes it impractical. We optimize for this runtime to sub-exponential time (for single query) and polytime time (using multiple-queries) in Section~\ref{sec:linearly_separable}. 
\vspace{-0.9ex}

\noindent\textbf{Role of Arithmetic Precision.} Our label inference attacks use number theoretic constructions with large integers and products of primes, which can render these attacks impractical to run (within a single query) on limited floating-point precision machines. We begin by observing that Definition~\ref{def:codomain_separability} does not take into account fixed arithmetic precision, which has an effect on separability by placing a bound on the resolution. For example, even if $f(\sigma_1,\theta) \neq f(\sigma_2,\theta)$, this difference may not be observable with only $\phi$ bits of precision,. We extend the notion of codomain separability in the finite precision model in Appendix~\ref{app:apa_fpa}, and present multi-query label inference attacks to recover all labels within fixed precision in Sections~\ref{sec:linearly_separable} and~\ref{sec:experiments}. For simplicity, we focus on inference attacks under arbitrary precision arithmetic in the main body of this paper.


%% file: linearly_separable.tex
\vspace{-0.5em}
Our main focus in this paper is on an important class of (binary) loss functions, which we refer to as \emph{linearly-decomposable}. These functions can be expressed as a sum of two terms: one dependent on the true labeling $\sigma$, and the other only on the prediction vector $\theta$. As we will see, this decomposition allows for an efficient construction of prediction vectors for robust label inference from such functions. We present only our main ideas here and defer all missing details to Appendix~\ref{app:linearly_separable}.

\begin{definition}
\label{def:linear_separable}
Let $g:[0,1] \to \mathbb{R}$ be some deterministic function. We say that a binary loss function $f:\{0,1\}^N \times (0,1)^N \to \mathbb{R}$ is linearly-decomposable if there exists some invertible function $g:[0,1] \to \mathbb{R}$ and some function $h:[0,1]^N \to \mathbb{R}$, such that:
\vspace{-0.5em}
\begin{align}
    \label{def:sqfree_linearly_separable}
    f(\sigma,\theta) &= h(\theta) + \sum_{i=1}^N \sigma_i g(\theta_i) = h(\theta) + \sum_{i:\sigma_i=1} g(\theta_i).
\end{align}
\end{definition}
\vspace{-0.5em}
This class of functions includes many commonly used loss functions in the ML literature. For example, all Bregman divergence-based binary loss functions satisfy the linear-decomposability property.
\begin{lemma}
\label{lem:bregman_is_linear}
Let $F:[0,1] \times [0,1] \to \mathbb{R}$ be a strongly convex function and $D_F(p,q) = F(p) - F(q) - \langle \nabla F(q),p - q\rangle$ be the Bregman divergence associated with $F$. Then, the corresponding loss function, defined as $f_F(\sigma,\theta) = \frac{1}{N}\sum_{i=1}^N D_F\paran{[\sigma_i,1-\sigma_i],[\theta_i,1-\theta_i]}$, is linearly-decomposable. 
\end{lemma}
Unlike the distance metrics for probability distributions, Bregman divergences does not require its inputs to be necessarily distributions~\citep{bregman1967relaxation}. To use these divergences as loss functions, we directly compare the outputs of the ML model with the point distribution from the ground truth labels, as we do in the definition of $f_F(\sigma,\theta)$ (similar to~\citep{liu2016clustering}). 

We also focus on a special class of linearly-decomposable functions for which the linear split is based purely on the labels: one term corresponding to datapoints with true label $0$, and the other for datapoints with true label $1$. More formally, if $g:[0,1] \to \mathbb{R}$ is some deterministic function, then we say that a binary loss function $f:\{0,1\}^N \times (0,1)^N \to \mathbb{R}$ is $g$-linearly-decomposable if it can be expressed as follows:
\begin{align}
\label{eq:linear_separable}
f(\sigma,\theta) &= \frac{1}{N}\paran{\sum_{i:\sigma_i=1} g(\theta_i) + \sum_{i:\sigma_i=0} g(1-\theta_i)}.    
\end{align}
In many cases, as we will see, functions in this sub-class are easier to analyze for codomain separability.
Observe that the KL-divergence loss (which reduces to binary cross-entropy or log-loss) is of this form, using $g(\theta_i) = -\ln \theta_i$. Some other
common examples of $g$-linearly-decomposable loss functions include the (i) Itakura-Saito divergence based loss, which uses $g(\theta_i) = 1/\theta_i + \ln \theta_i - 2$; (ii) Squared Euclidean loss, which can be expressed using $g(\theta_i) = (1-\theta_i)^2$; and, (iii) norm-like loss, which uses $g(\theta_i) = 1 + (\alpha-1)\theta_i^\alpha - \alpha \theta_i^{\alpha-1} + (\alpha-1)(1-\theta_i)^\alpha$ for some $\alpha \ge 2$. We provide detailed constructions for robust label inference from these particular loss functions in Appendix~\ref{app:linearly_separable}.

\vspace{-0.5em}
\subsection{Establishing Codomain Separability}
\vspace{-0.5em}
As argued in Section~\ref{sec:prelim}, the first step for label inference is to design prediction vectors using which the loss functions are sufficiently codomain-separable. We provide two different constructions of such prediction vectors for linearly-decomposable functions. Each construction uses a different set with distinct subset sums to ensure that the loss scores are in 1-1 correspondence with the set of all possible labelings. The first construction uses powers of 2, which follows naturally from the requirement of the $2\tau$ separation needed for $\tau$-robust label inference (as in Proposition~\ref{prop:bound_on_tau_for_label_inference}). We analyze multiple loss functions using this construction. Our second construction uses a set consisting of (log) primes, which enables us to perform robust label inference in sub-exponential time using results from number theory. We discuss these constructions in detail below.

\noindent\textbf{Construction 1:} As mentioned above, our first construction is based on the fact that the set $S_m = \{1,2,4,\dots,2^m\}$ has distinct subset sums. To see this, observe that each subset sum in $S_m$ is an integer whose binary representation (in reverse) is given by the \emph{bits} defined by the indices of elements contained in that subset. Moreover, since the difference between any two integers that can be represented this way is one, it also holds that the minimum subset sum difference is 1. Scaling each element of $S_m$ also scales the minimum difference as needed. The theorem below states our main result from this construction.
\begin{theorem}
\label{thm:linear_separable}
Let $g:[0,1] \to \mathbb{R}$ be some deterministic function and $f:\{0,1\}^N \times (0,1)^N \to \mathbb{R}$ be a loss function that is $g$-linearly-decomposable. Then, for any $\tau > 0$, the function $f$ is $2\tau$-codomain separable if there exists $\theta \in (0,1)^N$ so that $g(\theta_i) - g(1-\theta_i) > 2^iN\tau$ for all $i \in [N]$. If $\tau=0$, then setting $g(\theta_i) - g(1-\theta_i) > 0$ for all $i \in [N]$ suffices for \zero-codomain separability.
\end{theorem}

Based on this theorem, the prediction vectors can be constructed as follows: (1) Compute $x^*(y)$ as the solution to $g(x)-g(1-x) = y$; (2) If $x^*(y)$ exists, then set $\theta_i = x^*(2^iN\tau)$ for all $i \in [N]$. This vector $\theta$ can now be used for $\tau$-robust label inference from $f$ using~\LabelInf (see Section~\ref{sec:experiments} for our empirical analysis using this construction). The following corollaries follow for specific loss functions such as Itakura-Saito divergence loss, squared Euclidean, and norm-like divergence losses (see detailed proofs in Appendix~\ref{app:linearly_separable}). For simplicity, we discuss only the unnoised case in Corollary~\ref{cor:eucl}, for which it suffices to demonstrate \zero-codomain separability.


\begin{corollary} \label{cor:IS}
The Itakura-Saito divergence loss is $2\tau$-codomain separable with $\theta_i = (1+3^{2^{i}N\tau})^{-1}$.
\vspace*{-0.25em}
\end{corollary}



\begin{corollary} \label{cor:eucl}
The squared Euclidean loss  is \zero-codomain separable using $\theta_i = (1/2)\paran{1-\ln(p_i)/N}$, where $p_i$ is the $i$th prime number for $i \in [N]$. The norm-like divergence loss for $\alpha \geq 2$ is \zero-codomain separable using $\theta$ that satisfies $(1-\theta_i)^{\alpha-1}-\theta_i^{\alpha-1} = (\ln p_i)/(N\alpha)$.
\end{corollary}


\noindent\textbf{Construction 2:} Our main motivation for this second construction is to infer all labels within sub-exponential time. Starting from Definition~\ref{def:sqfree_linearly_separable}, we can
ensure $\Lambda_\theta(f) \geq 2\tau$ (as is needed for $\tau$-robust label inference from Proposition~\ref{prop:bound_on_tau_for_label_inference}) by setting $g(\theta_i) = 3P\tau  \ln p_i$ (if possible to do so within the domain of $g$), where $p_i$ is the $i^{th}$ prime number and $P=\prod_{i=1}^N p_i$. This particular choice of $g(\theta_i)$ ensures that each subset sum in the set $g(\theta)$ corresponds to the logarithm of a unique integer (using its prime factorization), and leads to the desired codomain separation as follows.
\begin{theorem}
\label{thm:linearly_decomposable_using_primes}
Let $f:\{0,1\}^N \times (0,1)^N \to \mathbb{R}$ be a loss function that is linearly-decomposable (Definition~\ref{def:sqfree_linearly_separable}). Let $p_i$ is the $i^{th}$ prime number and $P=\prod_{i=1}^N p_i$, is the product of the first $N$ primes. Then, for any $\tau > 0$, setting $g(\theta_i)=3P\tau  \ln p_i$ for loss functions in Equation~\ref{def:sqfree_linearly_separable} ensures that $\Lambda_\theta(f) \geq 2\tau$. If $\tau=0$, setting $g(\theta_i)=\ln p_i$ suffices for \zero-codomain separability.
\end{theorem}
We will now see how this choice of prime-based vector entries enable efficient label inference for any given $\tau$.
\vspace{-0.5em}
\subsection{Establishing Robust Label Inference}
\label{sec:label_inference_linear}
\vspace{-0.5em}
We discuss three approaches to recover the ground truth labels from the (perturbed) loss score obtained on prediction vectors we just designed. We distinguish each approach based on its runtime complexity and number of queries made to the server.

\noindent\textbf{Exponential Time Single-Query Inference.} The first approach is discussed in \LabelInf~\eqref{eq:labelinf}, which iterates over all $2^N$ possible labelings to find the one that is closest to the observed loss score. Both constructions 1 and 2 ensure that this algorithm always returns the true labeling within a single query to the server.

\noindent\textbf{Sub-Exponential Time Single-Query Inference.} To avoid the exhaustive search as above, we first substitute the prediction vector from Construction 2 in the expression for the loss function to obtain $f(\sigma,\theta) = h(\theta) + 3P\tau \ln \paran{\prod_{i:\sigma_i = 1}p_i}$.
Now, we know that due to codomain separation, the labeling that minimizes the distance between the observed loss $\ell$ and the true loss $f(\sigma,\theta)$ above is the true labeling $\sigma^*$:
\begin{align*}
    \sigma^* &= \arg\min_{\sigma \in \{0,1\}^N} \abs{\ell - \paran{h(\theta) + 3P\tau \ln \paran{\prod_{i:\sigma_i = 1}p_i}}} \\
    &= \arg\min_{\sigma \in \{0,1\}^N} \abs{\exp\parfrac{\ell - h(\theta)}{3P\tau} - \paran{\prod_{i:\sigma_i = 1}p_i}}.
\end{align*}

To obtain $\sigma^*$ from the equation above without using an exhaustive search over $\{0,1\}^N$, we observe that the expression inside the argmin essentially seeks a \emph{square-free} integer closest to some known real quantity. An integer is said to be squarefree if it has no repeated prime factors. This can be checked using the Booker-Hiary-Keating algorithm from~\citep{Booker_2015} in sub-exponential time. Once such an integer is obtained, its prime factors are in 1-1 correspondence with the indices in $\sigma^*$ that have label $1$. 
\vspace{-0.5em}
\SetKwInput{KwInput}{Input}
\SetKwInput{KwOutput}{Output}
\begin{algorithm}[ht]
\caption{\textsc{CloSqFree}$(x,m)$}\label{alg:closest_square_Free}
If $x \leq 2$ or $m=2$, then \textbf{return} 2.

Let $p_i$ be the $i^{th}$ prime number. If $p_1\cdots p_m \leq x$, then \textbf{return} $p_1\cdots p_m$.

\For{$k=0,1,2,\dots,\lfloor x \rfloor-1$}{
If $\lfloor x \rfloor - k \in \textsc{SqFree}(m)$, then \textbf{return} $\lfloor x \rfloor - k$.

If $\lceil x \rceil + k \in \textsc{SqFree}(m)$, then \textbf{return} $\lceil x \rceil + k$.
}
\end{algorithm}
The problem of label inference on linearly-decomposable loss functions is, thus, reduced to the calculation of the nearest square-free integer to a given real number. More concretely, the expression for recovering the true labeling can be written as follows:
$$\sigma^* = \left\{i: p_i\text{ divides }\textsc{CloSqFree}\paran{e^{\frac{\ell - h(\theta)}{3P\tau}},N}\right\},$$
where $\textsc{CloSqFree}(x,m) = \arg\min_{y \in \textsc{SqFree}(m)} |x-y|$ denotes the closest squarefree integer to $x$. The notation $\textsc{SqFree}(m)$ denotes the set of all square free integers whose largest prime factor is at most the $m^{th}$ prime number $p_m$. We outline an optimized version of $\textsc{CloSqFree}$ in Algorithm~\ref{alg:closest_square_Free}, which runs in $O(N\exp\paran{(\ln N)^{O(1)}})$ time. 








\noindent\textbf{Multi-Query Polynomial Time Attacks.} 
Loss functions that are $g$-linearly-decomposable also allow for an efficient multi-query label inference algorithm. In particular, we could have a trade-off between the ability to perform  multiple queries with faster computation times for solving the optimization problem in \LabelInf~\eqref{eq:labelinf}. 
To see this, observe that setting $\theta_i = 1/2$ gives $g(\theta_i) - g(1-\theta_i) = 0$. Thus, if we want to infer the first $M < N$ labels in a single query, we can set $\theta = [\theta_1,\dots,\theta_M,1/2,\dots,1/2]$, where 
$\theta_1,\dots,\theta_M$ are produced according to either Constructions 1 or 2. Using this $\theta$ will ensure that if $f(\sigma_1,\theta) = f(\sigma_2,\theta)$, then $\sigma_1[:M] = \sigma_2[:M]$.
After recovering the first $M$ labels, we can recover the next $M$ labels using $\theta = [1/2,\dots,1/2,\theta_{M+1},\dots,\theta_{2M},1/2,\dots]$, and so on. 

Observe that an $\lceil N/M \rceil$-query algorithm for robust label inference will require $O(N2^M/M)$ local computations by the adversary (using \LabelInf~\eqref{eq:labelinf} in each query). Thus, while the single query case required $O(2^N)$ computations, any multi-query algorithm using $M=O(\log N)$ requires only $O(\text{poly}(N))$ time.

%% file: log_loss_analysis.tex
We now show that the ideas of codomain separability also extend to the particular case of
multiclass cross-entropy loss and its variants. 
We provide an overview of our results and defer all missing details to Appendix~\ref{app:logloss}. 

\noindent\textbf{Multiclass Cross-Entropy Loss.} We first recall the definition of multiclass cross-entropy.
We assume $K \geq 2$ classes, and let $N$ and $\ZZZ_{K}$ denote the number of datapoints and the set of label classes, respectively. The $K$-ary cross-entropy loss on $\theta$ with respect to a labeling $\sigma \in \ZZZ_{K}^N$ is defined as follows:
	\begin{align}
	\label{eq:multi_class_log_loss}
		\Klogloss\paran{\sigma, \theta} := \frac{-1}{N}\sum_{i=1}^N \sum_{k=0}^{K-1} \ \Big([\sigma_i = k] \cdot \ln \theta_{i,k} \Big),
	\end{align}
 where $[\sigma_i = k] = 1$ if $\sigma_i = k$ and $0$, otherwise.
Here, $\theta \in \Theta = [0,1]^{N \times K}$ is a matrix of prediction probabilities, where the $i$th row is the vector of prediction probabilities $\theta_{i,0},\dots,\theta_{i,K-1}$ (with $\sum_{k \in \ZZZ_K} \theta_{i,k} = 1$) for the $i$th datapoint on classes $0,\dots,K-1$ respectively. 

With this definition, we can now describe our construction of a matrix $\theta \in [0,1]^{N \times K}$ that makes $\Klogloss$ function $2\tau $-codomain separable for any $\tau >0$ (i.e., $\Lambda_{\theta}(\Klogloss) \geq 2 \tau $). At a high level, we obtain the required codomain separability for $\Klogloss$ by splitting the loss into label dependent and independent terms, and designing the entries in the matrix $\theta$ in such a way that the expression reduces to a distinct integer in some set. This calibration allows us to control the minimum difference in the output of the cross-entropy loss on different labelings, which we can scale to the desired amount ($\geq 2\tau$) easily. The following theorem summarizes our construction.  
\begin{theorem}\label{thm:multi_class_proof}
Let $\tau > 0$. Define matrices $\vartheta, \theta \in \mathbb{R}^{N \times K}$ such that $\vartheta_{n,k} = 3^{\paran{2^{(n-1)K+k}N\tau}}$ and $\theta_{n,k} = \vartheta_{n,k} / \sum_{k=1}^K \vartheta_{n,k}$. Then, it holds that $\Klogloss$ is $2 \tau $-codomain separable using $\theta$. If $\tau = 0$, then using $\vartheta_{n,k} = 3^{\paran{2^{(n-1)K+k}}}$ ensures \zero-codomain separability.
\end{theorem}
Using Theorem~\ref{thm:multi_class_proof} and Proposition~\ref{prop:bound_on_tau_for_label_inference} implies that for these cross-entropy loss functions, the approach outlined in~\LabelInf~\eqref{eq:labelinf} succeeds in recovering all the labels when the loss scores are noised by less than $\tau$ in magnitude. We bring to the reader's attention the doubly-exponential nature of the entries used to construct the prediction vector in Theorem~\ref{thm:multi_class_proof}. This blowup is unfortunately unavoidable for constructing $\tau$-codomain separability, even for the binary case (see \citep[Theorem~7]{aggarwal2021icml}).

\noindent\textbf{Extensions of Cross-Entropy Loss.} Often in practice, when using the cross-entropy loss to assess the performance of ML models (like CNNs), it is common to apply an activation function (such as softmax or sigmoid) before the cross-entropy loss calculation. For example, a common idea in multiclass classification is to apply the softmax function (to convert any sequence of real outputs into a probability distribution) as:
\begin{align} \label{eq:softmax}
 \frac{-1}{N}\sum_{i=1}^N \sum_{k=0}^{K-1} \ \Big([\sigma_i = k] \cdot \ln \textsc{softmax}(\theta_{i,k}) \Big),
\end{align}
where $ \textsc{softmax}(\theta_{i,k}) = \exp(\theta_{i,k})/\sum_{j=1}^K \exp(\theta_{i,j})$.
To extend Theorem~\ref{thm:multi_class_proof} to this setting, we do the following: (1) Let $\theta \in [0,1]^{N \times K}$ be the matrix from Theorem~\ref{thm:multi_class_proof}. 
(2) For each $i \in [N]$, solve the following (fully specified) system of equations for $\theta'_{i,k}$'s: for all $k \in \ZZZ_K , i \in [N]$: $\exp\paran{\theta'_{i,k}}/\sum_{j=1}^K \exp\paran{\theta'_{i,j}} = \theta_{i,k}$.
Once $\theta'$ is obtained, since $\textsc{softmax}(\theta'_{i,k}) = \theta_{i,k}$, from Theorem~\ref{thm:multi_class_proof}, we get that softmax cross-entropy loss is $2\tau $-codomain separable using $\theta'$. A similar argument also works for Sigmoid cross entropy loss (Appendix~\ref{app:logloss}).



%% file: neural_network.tex
The results in previous sections highlight how prediction vectors ($\theta$'s) could be generated that succeed with $\tau$-robust label inference. 
This raises an interesting question of whether the prediction vectors utilized in these label inference attacks can actually be an output of a non-trivial benign looking ML model? We answer this question in the affirmative in this section.

\noindent\textbf{Setup.} Assume a classification problem on $K$ classes.
We will design a multi-layer feed-forward neural network $\textsc{MutNet}$ using following specifications.
(1) The input to $\textsc{MutNet}$ is a vector $\mathbf{v} \in \mathbb{R}^{d_1}$ with true label $\sigma_\mathbf{v} \in \ZZZ_K$.   (2) The output of $\textsc{MutNet}$ is a vector in $ (0,1)^{d_2}$ that represents an encoding of the prediction. 
Let $f: \ZZZ_K \times (0,1)^{d_2} \to \mathbb{R}$ be a  loss function.  
The goal of the network is to ensure that on any input $\mathbf{v}$ the network generates an output $\mathbf{u}_{m+1}$ such that $f$ admits $2\tau$-codomain separability using $\mathbf{u}_{m+1}$. Consequently, for any input $\mathbf{v}$, given a noisy value of $f(\sigma_{\mathbf{v}},\mathbf{u}_{m+1})$, an adversary can use \LabelInf to infer~$\sigma_{\mathbf{v}}$.



\noindent\textbf{Our Mutator Network.} We construct a 2-layer network $\textsc{MutNet}_{\mathbf{x}}:\mathbb{R}^{d_1} \to (0,1)^{d_2}$ that can convert any real vector $\mathbf{v} \in \mathbb{R}^{d_1}$ into any desired fixed vector $\mathbf{x} \in (0,1)^{d_2}$. The transformations we use in this network are the \textsc{ReLU} and \textsc{Sigmoid} activation functions: one layer of the former and one layer of the latter. Let $M_1\in \mathbb{R}^{d_1 \times d_1}$ and $M_2 \in \mathbb{R}^{d_1 \times d_2}$ be matrices such that all entries in $M_1$ are negative (the entries in $M_2$ can be arbitrary). Let $\mathbf{x}'$ be a vector such that $x'_i = \ln x_i/(1 - x_i)$. Then, for an input vector $\mathbf{u}_1 = \mathbf{v}$, the transformations in $\textsc{MutNet}_{\mathbf{x}}$ are as follows:
$
   \mathbf{u}_2 = \textsc{ReLU}\paran{\mathbf{v}^\top M_1},\ \mathbf{u}_3 = \textsc{Sigmoid}\paran{\mathbf{u}_2^\top M_2 + \mathbf{x}'}
$, where $\textsc{ReLU}$ and $\textsc{Sigmoid}$ are applied element-wise on their input vectors. Effectively, this construction inhibits the propagation of $\mathbf{v}$ by outputting the same vector $\mathbf{u}_3 = \mathbf{x}$ always.
By setting $\mathbf{x}$ to the desired (prediction vector) $\theta$ for $2\tau$-codomain separability on $f$, we get that $\abs{f(\sigma_\mathbf{v},\textsc{MutNet}_\theta(\mathbf{v})) - \ell} \leq \tau$ (Theorem~\ref{thm:nn}).
\begin{remark}
We note that  any neural network model can be modified to carry out our attack, as an adversary can replace the top layer of any neural network model with the above construction. This highlights the versatility of our attack.
\end{remark}

%% file: empirical_analysis.tex
We now present an empirical evaluation of our label inference attacks. We analyze the following datasets: 
\input{single_query_plot}
\input{plots}
\vspace{-1em}
\begin{CompactItemize}
    \item {\bf{Titanic}}~\citep{titanic}: a binary classification dataset (2201 rows) on the survival status of passengers.
    \item {\bf{IMDB}}~\citep{imdb}: a binary classification dataset (25000 rows) on the movie reviews.
    \item {\bf{Satellite}}~\citep{satellite}: a six-class classification dataset (6430 rows) on the satellite images of soil.
    \item {\bf{MNIST}}~\citep{mnist}: a ten-class classification dataset (70000 rows) on handwritten digits.
    \item {\bf{CIFAR}}~\citep{cifar10}: a ten-class classification dataset (60000 rows) on color images.
\end{CompactItemize}
As our attacks construct prediction vectors that are independent of the dataset contents, we ignore the dataset features in our experiments, but all the labels of the dataset are considered for the attack\footnote{All experiments are run on a 64-bit machine with 2.6GHz 6-Core processor, using the standard IEEE-754 double precision format. For reproducibility, the code is included as part of the supplementary material.}. We consider four common loss functions arising from Sections~\ref{sec:linearly_separable} and \ref{sec:logloss}, two of which are multiclass losses (multiclass cross-entropy (\Klogloss~\eqref{eq:multi_class_log_loss}) and softmax cross-entropy~\eqref{eq:softmax}), and the other two are binary losses (binary Itakura-Saito (\textsc{ISLoss}~\eqref{eq:IS}) and sigmoid cross-entropy~\eqref{eq:sigmoid_definition}). 
As a baseline, for the binary labeled dataset (Titanic) with the (plain) cross-entropy loss, we implemented the label inference attack of~\citep{aggarwal2021icml}.
We also present additional experimental results in Appendix~\ref{app:experiments}.

We start with the distinction between our experiments and the approach outlined in \LabelInf, which is presented in the arbitrary precision model.
For \Klogloss, simulating \LabelInf on a finite precision machine must be able to differentiate $\min_{i,k}\theta_{i,k}$ from 0 (or else the label inference will be ambiguous). A rough analysis (from Theorem~\ref{thm:multi_class_proof}) gives that 
$\phi=\Omega(2^{NK}N\tau)$ bits are required to make this distinction. This bound hides constant factors, but gives an idea of how arithmetic precision  plays a role in our experiments. 

Figure~\ref{fig:singlequery} shows the number of datapoints $N$ recovered by \LabelInf as we increase the noise for \textsc{ISLoss}, \Klogloss, and  softmax cross-entropy loss. We sample 1000 random sets of labels each of length $N$ from the dataset here. At error magnitude $\tau = 1$, the noise is comparable to the actual loss function values computed. We measure accuracy as the percentage of labels correctly inferred out of $N$.
Figure~\ref{fig:singlequery}(a) shows that the number of datapoints ($N$) for which 100\% label inference accuracy is achieved with \textsc{ISLoss} as we vary noise magnitude. As expected at lower noise magnitude the attacks can extract more labels correctly and this quantity decreases as the noise magnitude ($\tau$) increases. Note that as mentioned, our attacks are independent of the dataset feature set, therefore the recovery performance is same on both Titanic and IMDB, as they are both binary labeled.
Figure~\ref{fig:singlequery}(b) shows the number of datapoints ($N$) for which 100\% label inference accuracy is achieved with \Klogloss. Here, the results differ (except for MNIST and CIFAR) because the datasets have different number of classes $K$.
At the same noise magnitude, the accuracy also drops as the number of classes ($K$) increases from 2 (in Titanic) to 10 (in MNIST and CIFAR), which is also expected. These happen because of the dependence on number of datapoints and number of classes in our prediction vector construction (Theorem~\ref{thm:multi_class_proof}), which given fixed machine precision runs into representation issues. 
We note that the construction from~\citep{aggarwal2021icml} on the Titanic dataset (the only case where it is applicable), works slightly better (especially at lower $\tau$ values) than ours due to a small difference in the exponent: it holds that $\min_{i \in [N]}(-\ln \theta_{i}) = \Omega(2^N N\tau)$ in their paper, but $\Omega(2^{2N}N\tau)$ when using Theorem~\ref{thm:multi_class_proof}.



Figures~\ref{fig:singlequery}(c)-(d) show the number of datapoints on which \LabelInf achieves at least 50\% accuracy for \Klogloss and softmax cross-entropy losses respectively. We notice this number is smaller for softmax cross-entropy loss, which is also not surprising. Through a similar argument as that above for the number of bits required for cross-entropy loss, one can show that computing the softmax cross-entropy loss will require an additional $\Omega(NK + \ln(N\tau))$ bits (see the discussion in Appendix~\ref{app:experiments} for details). This additional requirement further constraints the number of labels that can be recovered with softmax cross-entropy loss.

We also examine a multi-query label inference algorithm in Figure \ref{fig:multiquery}. For these plots, we simulated \LabelInf on $M < N$ datapoints at a time (instead of all), to obtain a total of $\lceil N/M\rceil$ queries. The idea is to use Figure~\ref{fig:singlequery} to determine the maximum number of labels that can be correctly inferred in a single query for a given noise level, and then perform label inference on only those many datapoints at a time.  As expected, we observed that the accuracy increases with the number of queries: for \textsc{ISLoss} on Titanic, we achieved 100\% accuracy using $M \geq 220$ with $\tau=0.0001$, and $M \geq 1100$ with $\tau = 1$. The accuracy is again lower for the softmax case again due to reasons mentioned above. Additional experimental results on \Klogloss and sigmoid cross-entropy losses are included in Appendix~\ref{app:experiments}.

%% file: single_query_plot.tex
\begin{figure*}[t]
    \centering
    \vspace{-0.5em}
    \subfloat[]{\includegraphics[width=120pt]{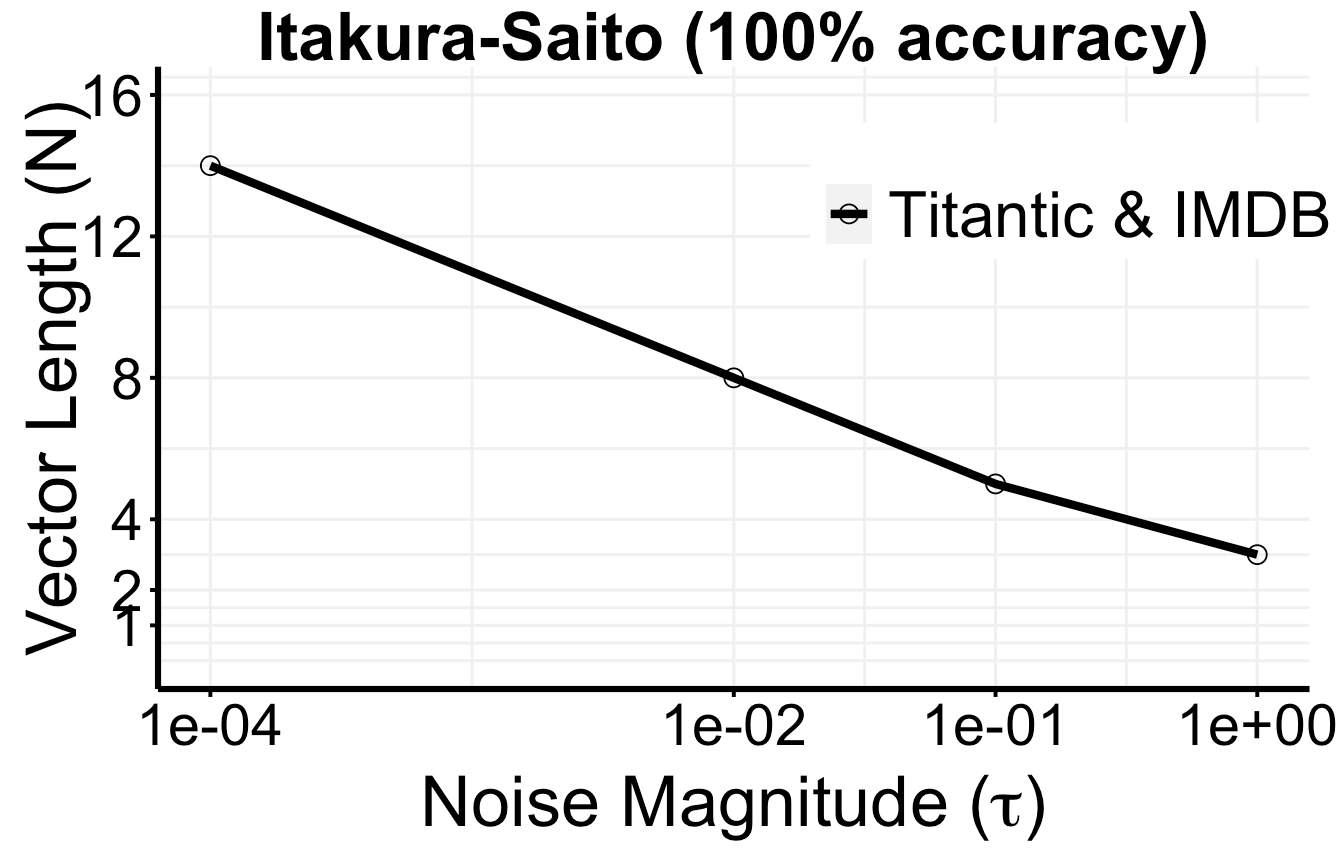}}
    \hspace{-0.1em}
    \subfloat[]{\includegraphics[width=120pt]{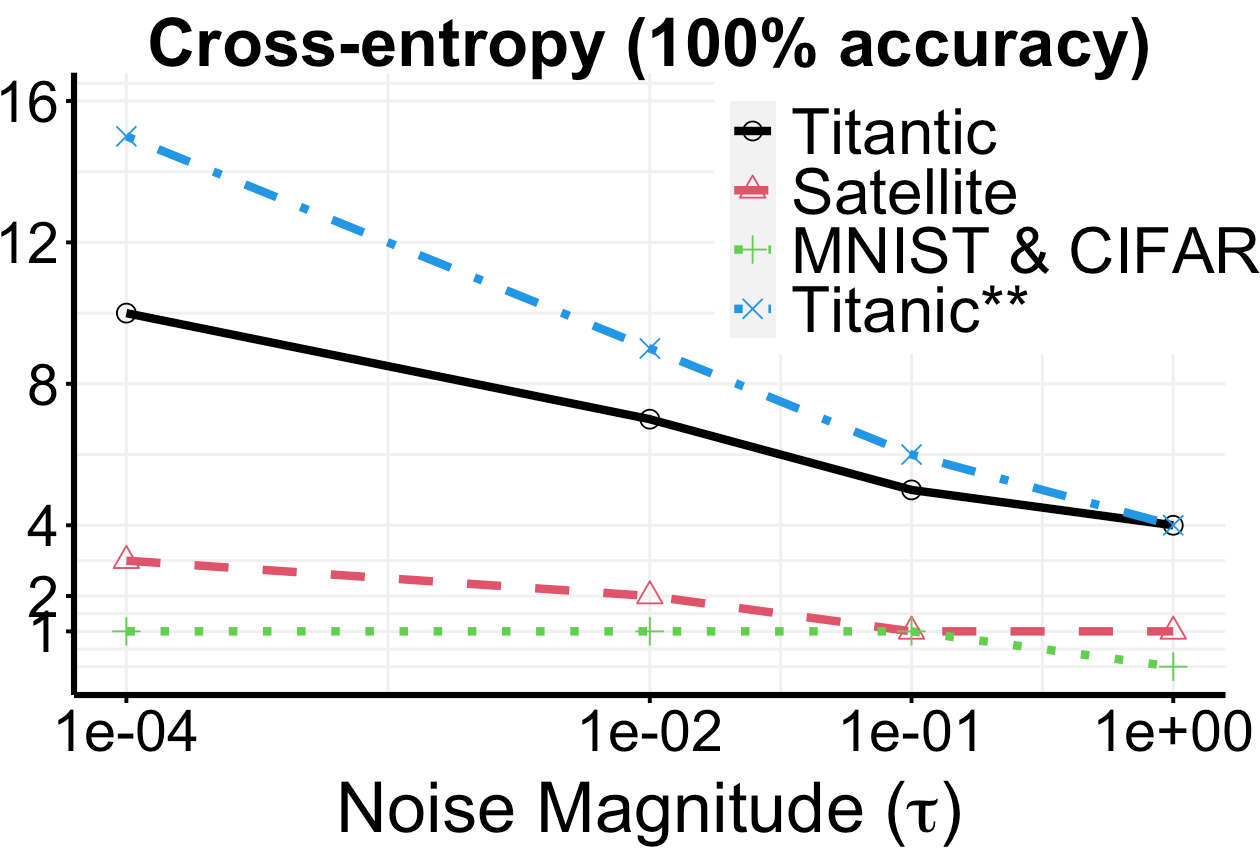}}
    \hspace{-0.1em}
    \subfloat[]{\includegraphics[width=120pt]{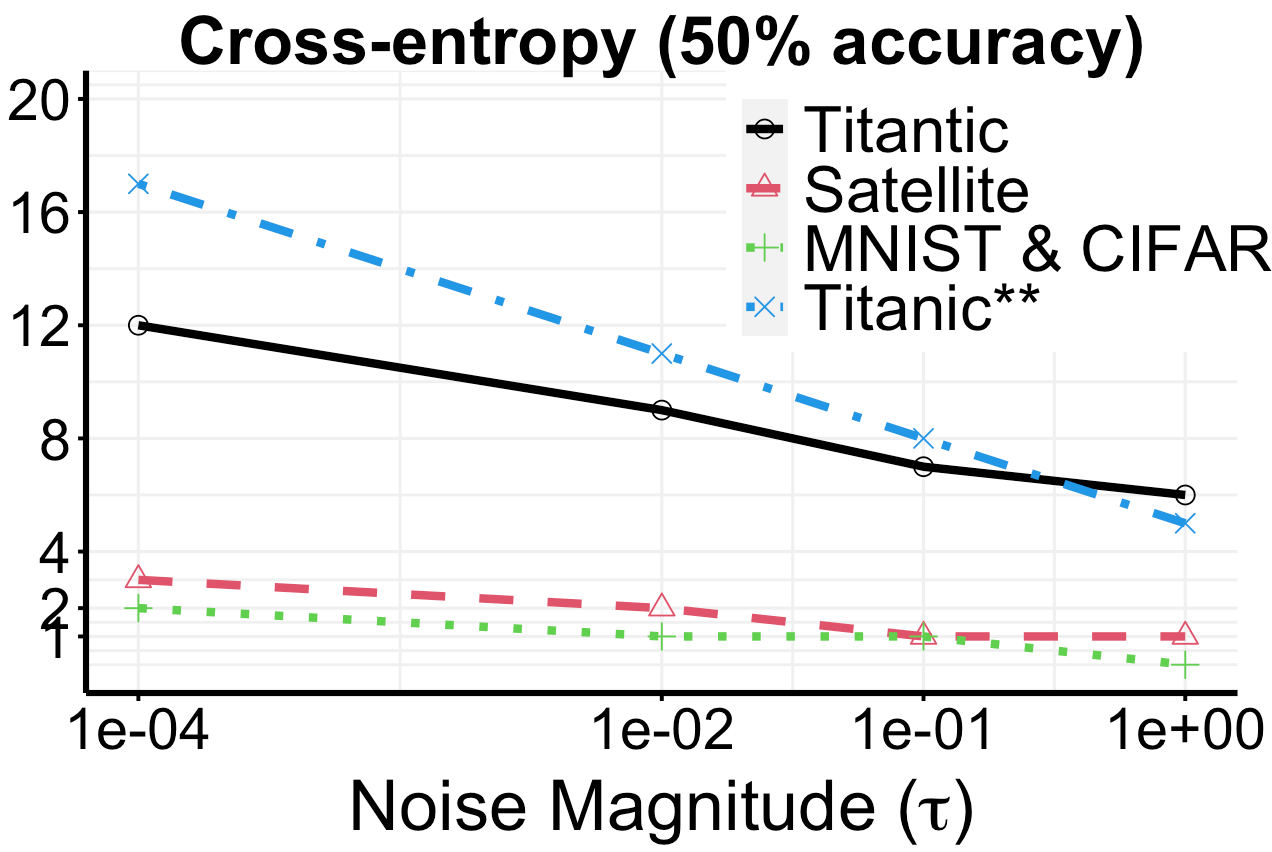}} 
    \hspace{-0.1em}
    \subfloat[]{\includegraphics[width=120pt]{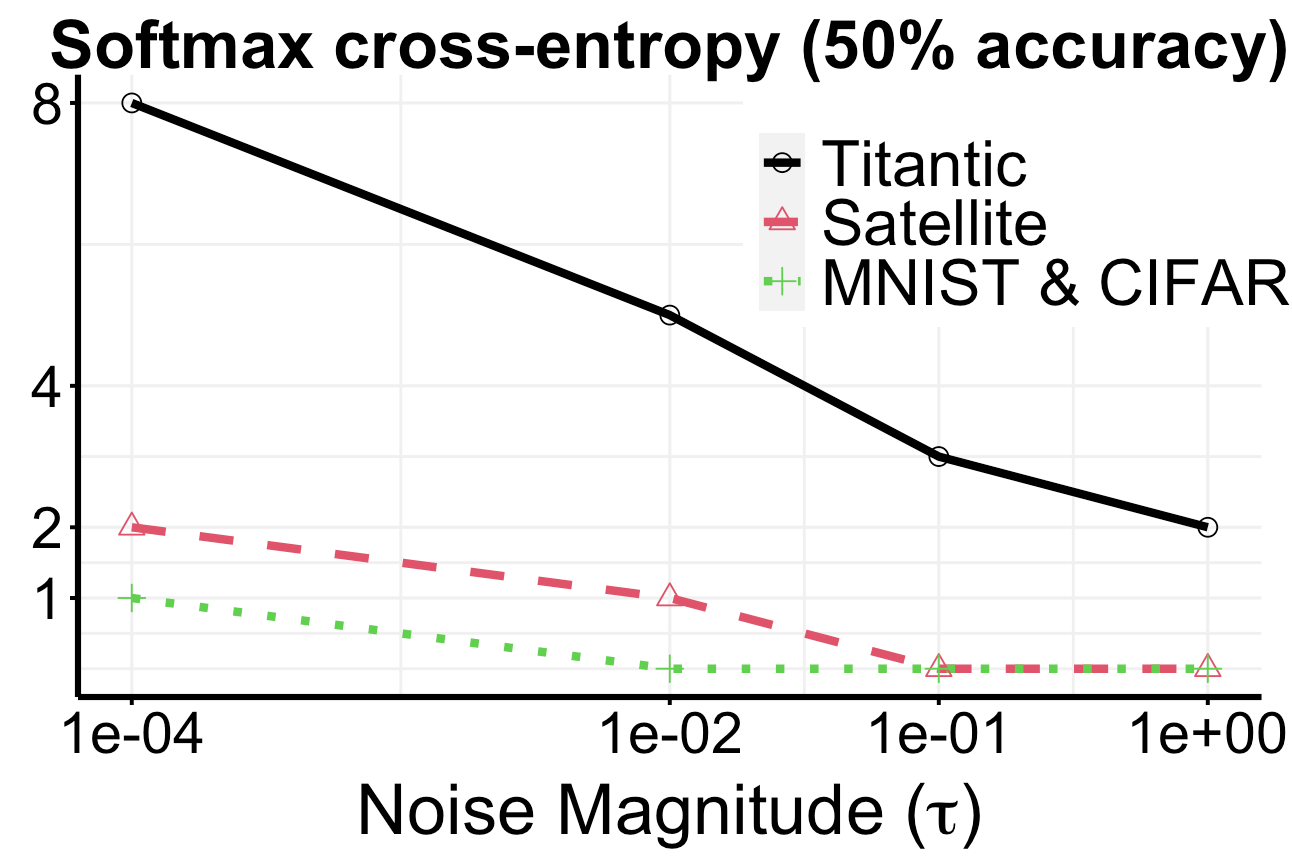}}
    \vspace*{-1ex}
    \caption{Results for single-query label inference. The Y-axis in Figures (a) and (b) represents the number of datapoints that can be recovered with 100\% accuracy, while for Figures (c) and (d), it represents the number of datapoints that can be recovered at 50\% accuracy, i.e., we recover at least this length vector accurately in at least 50\% of the 1000 runs.
    The max computation time per inference attack is roughly 10 seconds.} 
    \label{fig:singlequery}
    \vspace*{-1em}
\end{figure*}

%% file: plots.tex
\captionsetup[subfloat]{labelformat=empty}
\begin{figure*}[t]
    \centering
    \vspace{-0.5em}
    \subfloat[]{\includegraphics[width=120pt]{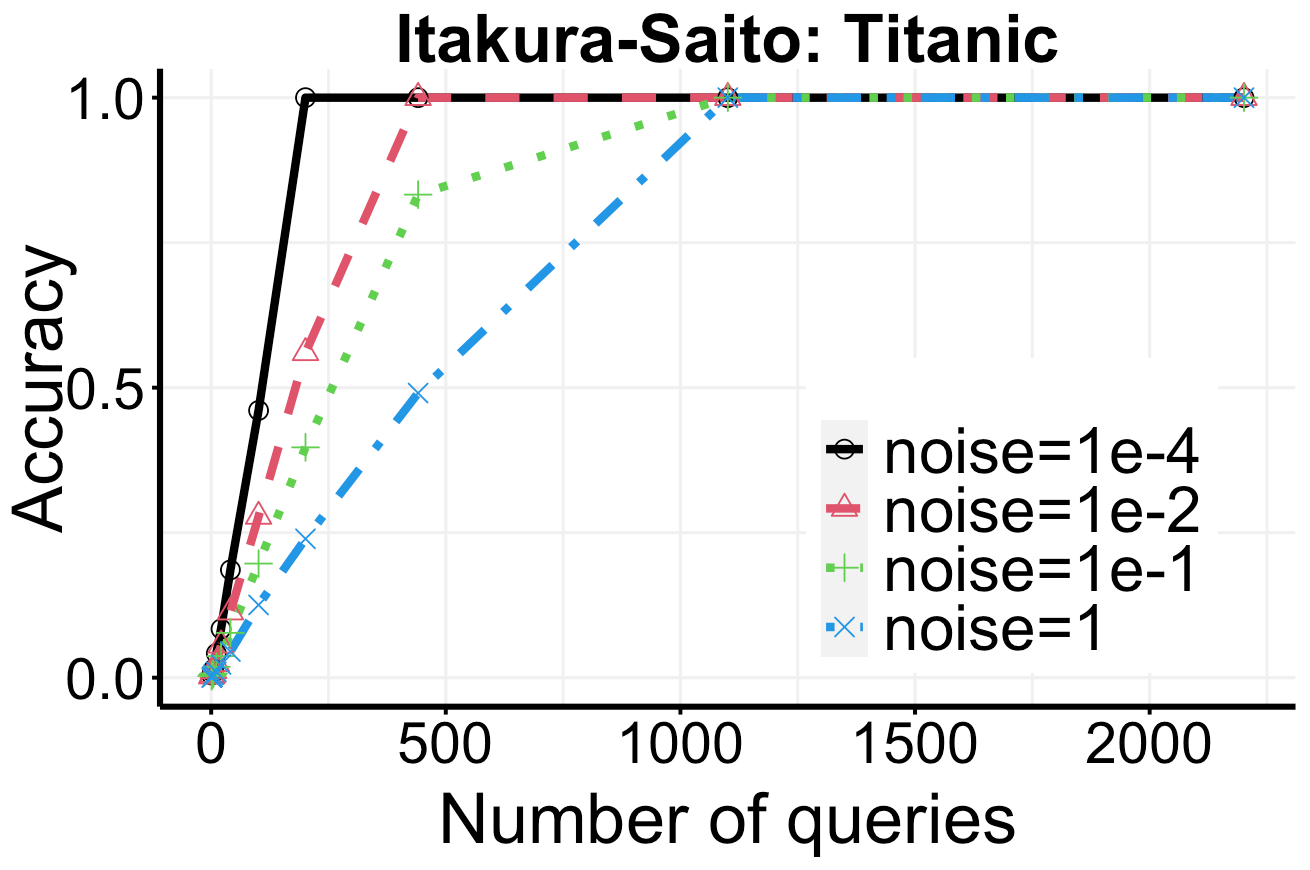}} 
     \hspace{0.1em}
     \subfloat[]{\includegraphics[width=120pt]{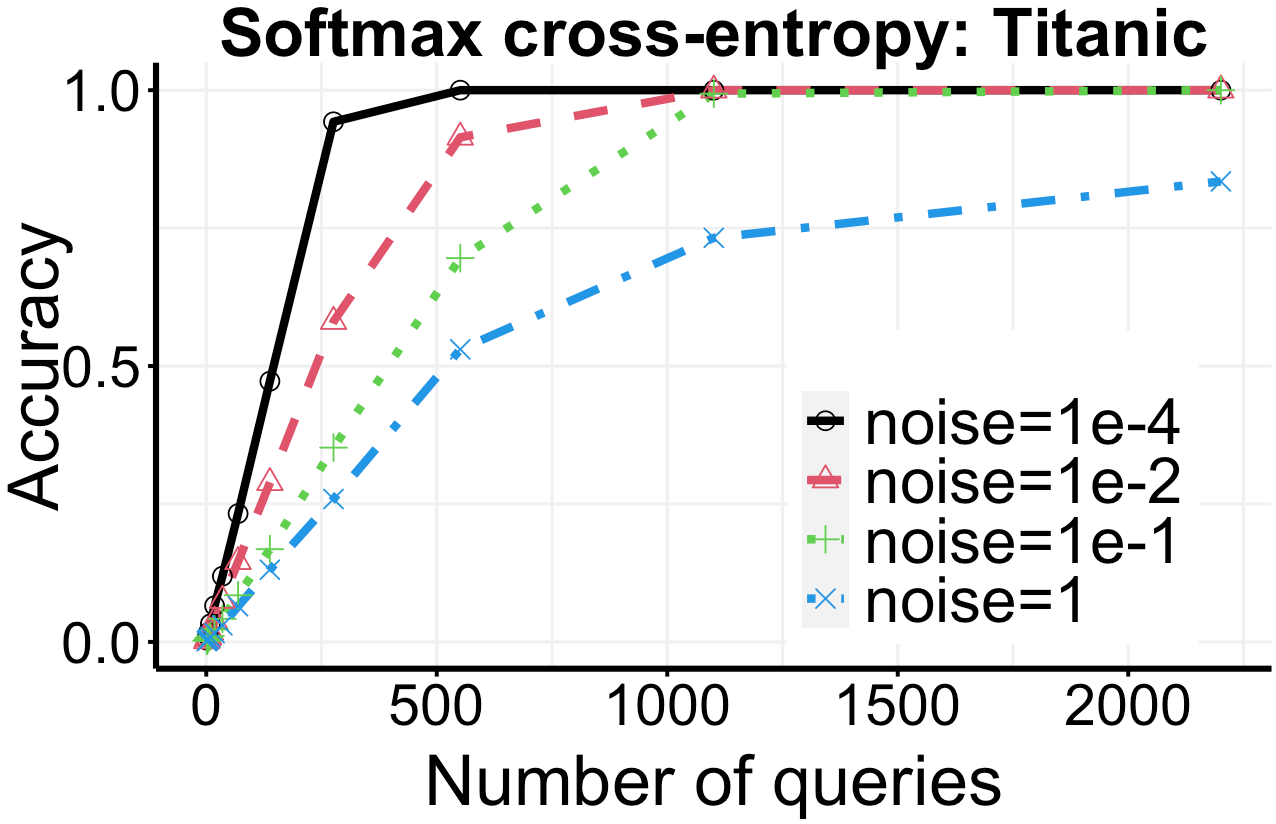}}
     \hspace{0.1em}
    \subfloat[]{\includegraphics[width=120pt]{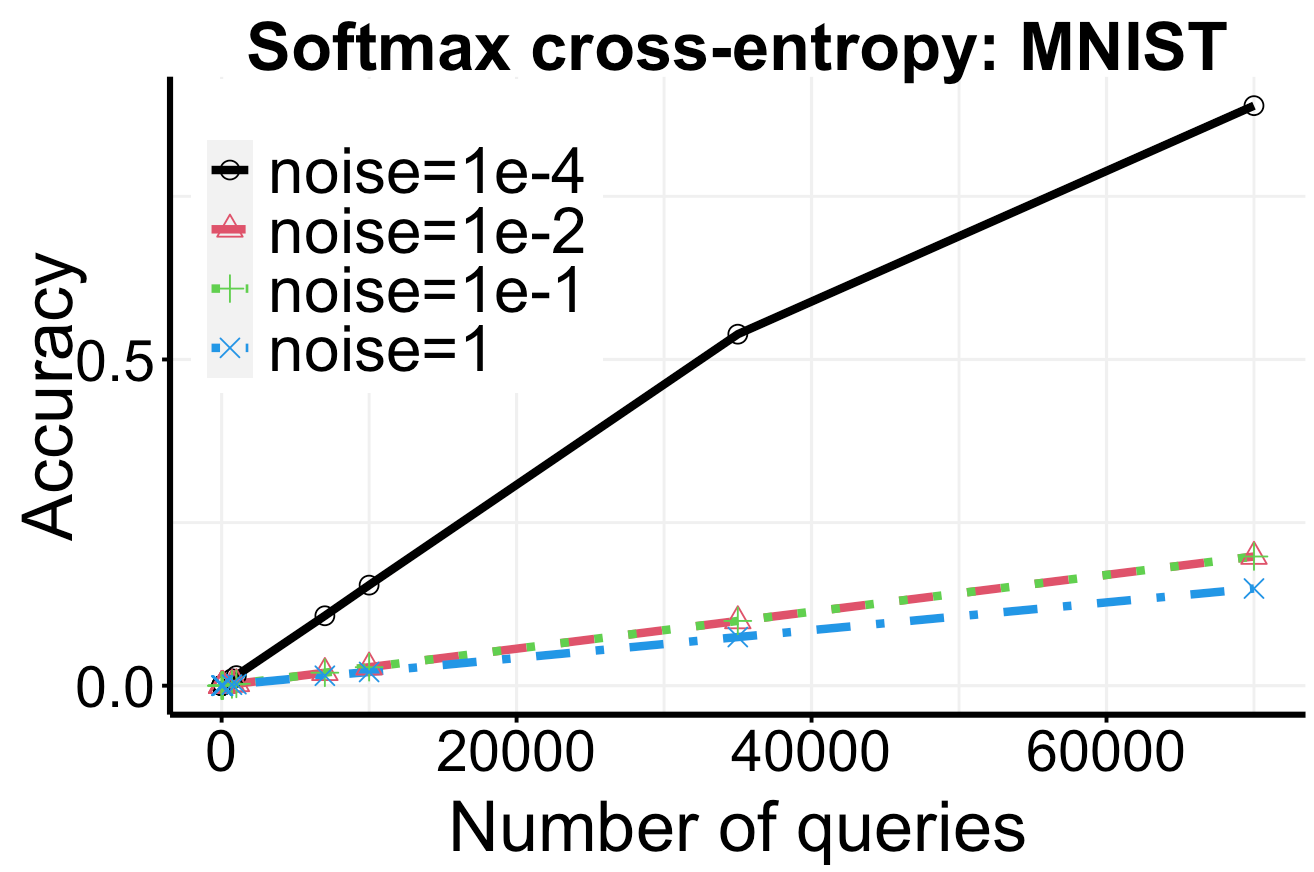}} 
    \\[-3ex]
     \vspace{-0.8em}
    \subfloat[]{\includegraphics[width=120pt]{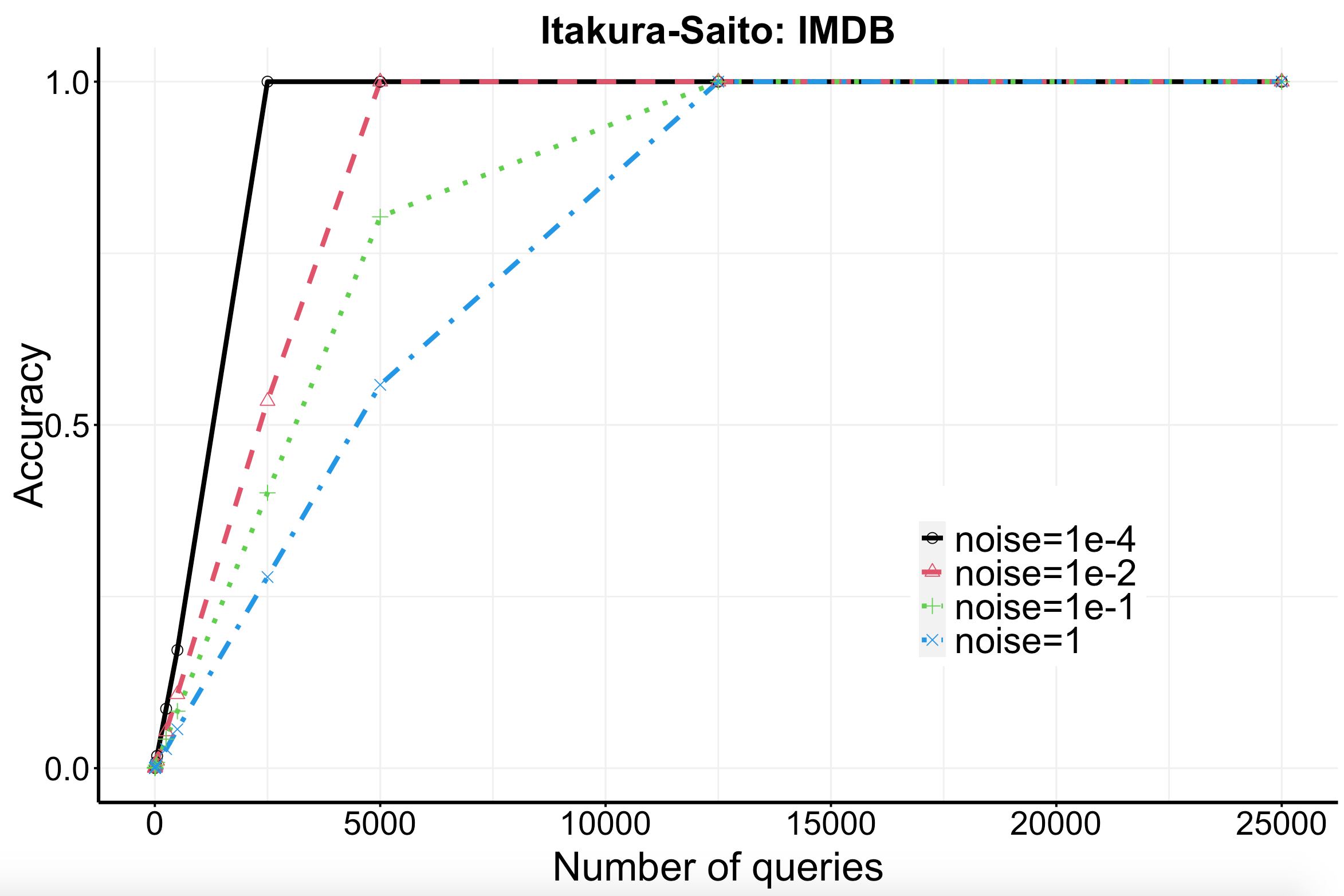}} 
    \hspace{0.1em}
    \subfloat[]{\includegraphics[width=120pt]{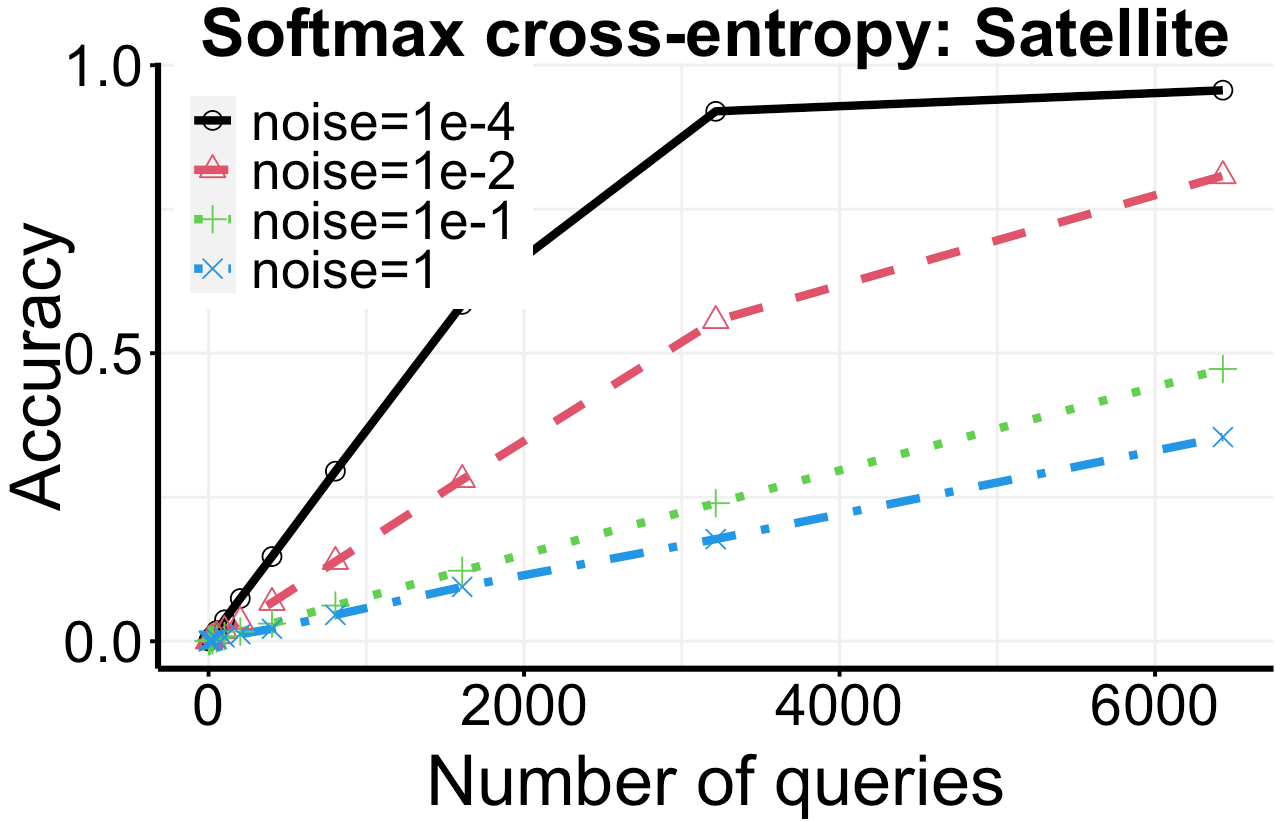}}  
    \hspace{0.1em}
    \subfloat[]{\includegraphics[width=120pt]{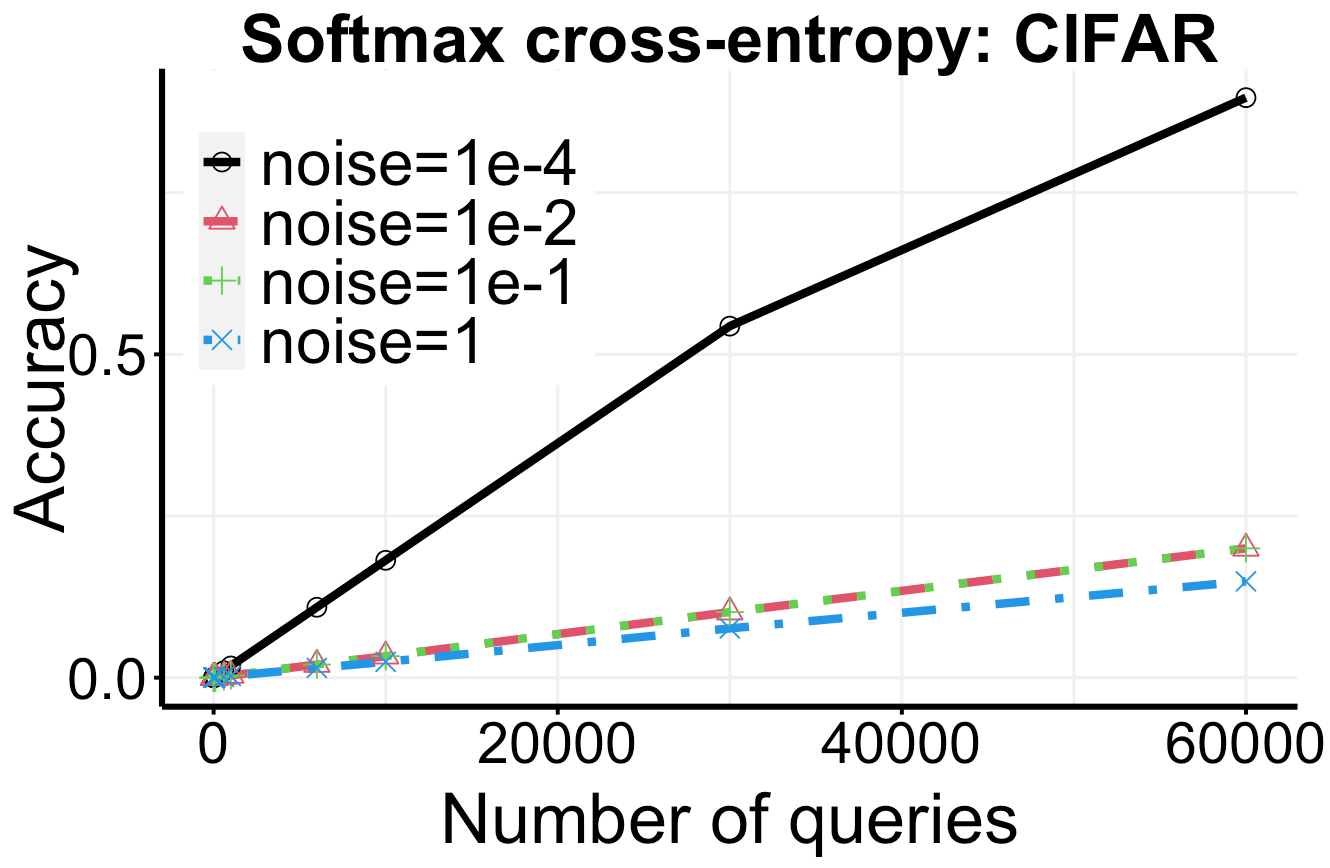}} 
    \vspace*{-4ex}
    \caption{Label reconstruction accuracy with the multi-query label inference attack. As discussed in the text, for a given loss-function, we expect similar plots for any two datasets with same number of classes, e.g., Itakura-Saito for Titanic and IMDB, Softmax cross-entropy for MNIST and CIFAR. }
    \label{fig:multiquery}
    \vspace{-1.7em}
\end{figure*}

%% file: conclusion.tex
In this paper, we demonstrated how a large class of common ML loss functions can be exploited to recover the unknown test labels. Our attacks, based on tools from analytical number theory, succeed with provable guarantees, even when provided with noisy loss function values. Our investigation also highlights the role of number of queries and arithmetic precision. We also demonstrate how our attack could be carried out using a simple augmentation to any neural network model, making it look benign. 




%% file: supplement.tex
\onecolumn
\aistatstitle{Appendix: Reconstructing Test Labels from Noisy Loss Functions}
\section{Computational Hardness of Codomain Separability}\label{app:hardness_results}
We show that that determining the codomain separability is co-NP-hard (see~\citep{arora2009computational} for a definition of co-NP-hard). We establish the result for the weaker notion of \zero-codomain separability (Definition~\ref{def:codomain_separability}) and restrict ourselves to functions of the form $f: \{0,1\}^N \times \ZZZ_{+}^N \to \mathbb{R}$, where the decision problem is to determine whether there exists a $\theta \in \ZZZ_{+}^N$ such that $f$ is \zero-codomain separable using $\theta$. We denote this decision problem by \textsc{Codomain-Sep}. Note that this automatically implies that determining the $\tau$-codomain separability of such functions is also co-NP-hard.

Let $\Phi(x_1,\dots,x_N)$ be a Boolean formula over $N$ variables in the 3-CNF form $\Phi(x_1,\dots,x_N) = C_1 \land \dots \land C_N$, where $C_i$ are disjunctive clauses containing 3 literals each. We say that $\Phi(x_1,\dots,x_N)$ is an \emph{almost tautology} if there are at least $2^N-1$ satisfying assignments for $\Phi$. In other words, there is at most one assignment of the variables that makes $\Phi(x_1,\dots,x_N)$ false. Define \textsc{Almost-Tautology} to be the problem of determining if a given Boolean formula is an almost tautology. 

\begin{lemma}
\label{lem:almost_taut_conpcomplete}
\textsc{Almost-Tautology} is co-NP-complete.
\end{lemma}
\begin{proof}
We first show that \textsc{Almost-Tautology} is in co-NP, i.e., any certificate that $\Phi$ is not an almost-tautology can be checked in polynomial time. To see this, observe that such a certificate must contain at least two distinct assignments for which $\Phi$ is not satisfied, which can be verified efficiently. 

To show that \textsc{Almost-Tautology} is co-NP Hard, there are two cases: either $\Phi$ is a tautology, or there is exactly one assignment that does not satisfy $\Phi$. Deciding the former is co-NP-hard~\citep{arora2009computational}. For the latter, consider the logical negation $\lnot \Phi(x_1,\dots,x_N)$. If $\Phi$ is an almost tautology (but not a tautology), then there is exactly one satisfying assignment for $\lnot \Phi$. This is the same as the Unique-SAT problem, in which we determine if a Boolean formula has a unique solution. This problem is also co-NP-hard~\citep{papadimitriou1982complexity,blass1982unique}, and hence, \textsc{Almost-Tautology} is co-NP-hard.
\end{proof}

We start with an arbitrary Boolean formula $\Phi(x_1,\dots,x_N)$. For $\theta \in \mathbb{Z}_+^N$, define the following function: $$f_{\Phi}(\sigma,\theta) = \hat{C}_1^{\theta_1}\cdots \hat{C}_N^{\theta_N} p_1^{\sigma_1}\cdots p_N^{\sigma_N},$$where $p_1,\dots,p_N \geq 5$ are distinct prime numbers, and $\hat{C}_i = C_i(\sigma)$ in the additive form (i.e., mapping all variables to $\{0,1\}^N$ by representing all negative literals $\bar x_i$ as $1 - x_i$, and converting all disjunctions to addition). For example, if $C_1 = (x_3 \lor \bar x_4 \lor x_6)$, then $\hat{C}_1 = C_1(\sigma) = \sigma_3 + (1 - \sigma_4) + \sigma_6$. Another example is as follows: assume $\Phi(x_1,x_2,x_3) = (\bar x_1 \lor x_2 \lor x_3) \land (x_1 \lor x_2 \lor \bar x_3)$. Then, first repeat the third clause in $\Phi$ to make the number of clauses the same as the number of variables to obtain $\Phi(x_1,x_2,x_3) = (\bar x_1 \lor x_2 \lor x_3) \land (x_1 \lor x_2 \lor \bar x_3) \land (x_1 \lor x_2 \lor \bar x_3)$. Now, the corresponding function can be written as follows:
\begin{align}
\label{eq:conpcomplete}
    f_\Phi(\sigma,\theta) = 5^{\sigma_1}7^{\sigma_2}11^{\sigma_3}(1-\sigma_1 + \sigma_2 + \sigma_3)^{\theta_1} (1 + \sigma_1 + \sigma_2 - \sigma_3)^{\theta_2 + \theta_3}.
\end{align}
We now prove our hardness result for \zero-codomain separability using this reduction.

\begin{lemma}
\label{lem:conp_helper}
For any $\theta \in \mathbb{Z}_+^N$ and distinct $\sigma_1,\sigma_2 \in \{0,1\}^N$, it holds that $f(\sigma_1,\theta) \neq f(\sigma_2,\theta)$ if and only if at least one of $\sigma_1$ or $\sigma_2$ satisfies $\Phi$. Here, we abuse the notation to represent the Booleans True and False by $1$ and $0$, respectively.
\end{lemma}
\begin{proof}
Observe that for any distinct $\sigma_1,\sigma_2 \in \{0,1\}^N$, we have $f(\sigma_1,\theta) = f(\sigma_2,\theta)$ only when there is some set of clauses $\hat{C}_i$ and $\hat{C}_j$ such that $\hat{C}_i(\sigma_1) = \hat{C}_i(\sigma_2) = 0$, i.e. they are unsatisfied by $\sigma_1$ and $\sigma_2$, respectively. This is because the product of the primes satisfies $\prod_{i=1}^N p_i^{\sigma_1(i)} \neq \prod_{j=1}^N p_j^{\sigma_2(j)}$ (since both products differ in at least one prime -- the one corresponding to the index of the element at which $\sigma_1$ and $\sigma_2$ differ). The lemma statement then follows from the contrapositive of this result.  
\end{proof}

\begin{theorem} \label{thm:conp}
\textsc{Codomain-Sep} is co-NP-hard.
\end{theorem}
\begin{proof}
We prove this by demonstrating a Karp reduction from the \textsc{Almost-Tautology} problem, which we showed is co-NP-complete in Lemma~\ref{lem:almost_taut_conpcomplete}. Let $\Phi(x_1,\dots,x_N)$ be an arbitrary Boolean formula and $f_\Phi(\sigma,\theta)$ be the corresponding function from~\eqref{eq:conpcomplete}. Now, for $\Phi$ to be an almost tautology, there can at most one unsatisfying solution: (1) if there are no unsatisfying solutions, then for any Boolean assignment of $x_1,\dots,x_N$, all clauses in $\Phi(x_1,\dots,x_N)$ must be satisfied and hence, from Lemma~\ref{lem:conp_helper}, the value of $f_\Phi(\sigma,\theta)$ must also be distinct for all $\sigma$, implying that $f_\Phi$ is \zero-codomain separable; (2) if there is an unsatisfying assignment, then the value of $f_\Phi(\sigma,\theta)$ at this assignment must be zero, and it is non-zero at all other assignments. This also implies that $f_\Phi$ is \zero-codomain separable. Lastly, let $\theta$ be a vector with respect to which $f_\Phi$ is \zero-codomain separable. This immediately implies that $\Phi(x_1,\dots,x_N)$ must be an almost tautology --- if not, then either one of (1) or (2) must be false since there will be at least two unsatisfying assignments in this case, implying that the function $f_\Phi$ will be zero on at least two inputs.
\end{proof}

\section{Missing Details from Section~\ref{sec:prelim}} \label{app:prelim}
The following proposition states the connection between $\tau$-robust label inference and $\tau$-codomain separability. This connection, for the specific case of binary cross-entropy loss, was also noted by~\citep{aggarwal2021icml}.

\noindent\textbf{Restatement of Proposition~\ref{prop:bound_on_tau_for_label_inference}.} \emph{A function $f$ admits $\tau$-robust label inference using $\theta \in \Theta$ if and only if $\Lambda_{\theta}(f) \geq 2\tau$.}
\begin{proof}
We start with one direction and show that if we can do label inference (there exists algorithm $\mathcal{A}$ in Definition~\ref{def:robust_label_inference}), then $\Lambda_{\theta}(f) \ge 2\tau$ must hold. We prove this by contradiction. The idea is to construct a score from which a unique labeling cannot be unambiguously derived. Without loss of generality, let $\sigma_1, \sigma_2$ be two distinct labelings for which $0 < f(\sigma_2,\theta) - f(\sigma_1,\theta) < 2\tau$. It follows that $f(\sigma_2,\theta) - \tau < f(\sigma_1,\theta) + \tau$. Now, let $\ell = \paran{f(\sigma_1,\theta)+f(\sigma_2,\theta)}/2$ and $x = \ell -  f(\sigma_1,\theta)$.
Clearly, $x < \tau$. Similarly, $f(\sigma_2,\theta) - \ell < \tau$. In other words, $\ell$ could be generated by a $\tau$ magnitude perturbation to both $f(\sigma_1,\theta)$ and $f(\sigma_2,\theta)$ (with $\sigma_1 \neq \sigma_2$). Therefore, there can exists no algorithm $\mathcal{A}$ that, given just $\ell$, can recover whether the true label is $\sigma_1$ or $\sigma_2$ (i.e., no $\mathcal{A}$ can succeed with $\tau$-robust label inference). This is a contradiction, therefore, $\Lambda_{\theta}(f) = \min_{\sigma_1,\sigma_2}\abs{f(\sigma_2,\theta) - f(\sigma_1,\theta)} \geq 2\tau$.

For the other direction, let $\ell$ be as given in Definition~\ref{def:robust_label_inference} with $\abs{f(\sigma^\star,\theta) - \ell} < \tau$. By triangle inequality it follows that if $\Lambda_{\theta}(f) \geq 2\tau$, then $\abs{\ell - f(\sigma^\star,\theta)} < \min_{\sigma \in \ZZZ_K^N \setminus \sigma^\star} \abs{\ell-f(\sigma,\theta)}$ (i.e., addition of any noise less than $\tau$ in  magnitude will maintain the invariant that the noised score is closest to the score on the true labeling).
Hence, solving  $\arg\min_{\sigma\in \ZZZ_K^N} \abs{f(\sigma,\theta) -\ell}$ will return the true label $\sigma^\star$.
\end{proof}

\subsection{Separability in Arbitrary Precision vs.\ Finite Floating-Point Precision} 
\label{app:apa_fpa}
One important consideration in our label inference attacks is the precision of arithmetic that is required at the adversary. In this context, there are two natural models of arithmetic computation: a) arbitrary precision and b) finite floating-point precision. Arbitrary precision arithmetic model allows precise arithmetic results even with very large numbers. In the floating-point precision model, the arithmetic is constrained by limited precision. An example of the floating-point precision model is the commonly used IEEE-754 double precision standard. Designing algorithms for standard arithmetic in both these models have been studied extensively~\citep{knuth2014art,brent2010modern}. We refer to the arbitrary precision arithmetic model as $\mathsf{APA}$ and floating point  arithmetic model with $\phi$ bits as $\mathsf{FPA}(\phi)$. For ease of discussion, we assume that in the $\mathsf{FPA}(\phi)$ model, we have $1$ bit for sign, $(\phi-1)/2$ bits for the exponent and $(\phi-1)/2$ bits for the fractional part (mantissa). This assumption can be relaxed to accommodate $\phi_a > 0$ bits for the exponent and $\phi_b > 0$ bits for the fractional part where $\phi_a + \phi_b = \phi-1$. Furthermore, for any loss function $f$ in the $\mathsf{APA}$ model, we denote by $f_\phi$ the algorithm that computes $f$ on a machine with an instruction set for performing computations within these $\phi$ bits of precision.

We begin by observing that Definition~\ref{def:codomain_separability} does not take into account fixed arithmetic precision (i.e., deals with the case where we have arbitrary precision arithmetic). Finite floating-point precision has an effect on separability, as bits of precision places a bound on the resolution. For example, even if $f(\sigma_1,\theta) \neq f(\sigma_2,\theta)$ in the $\mathsf{APA}$ model, with only $\phi$ bits of precision, this difference may not be observable. This leads to notion of separability in the $\mathsf{FPA}(\phi)$ model.

\begin{definition} [$\tau$-codomain Separability in the $\textsc{FPA}(\phi)$ model]
\label{def:FPA_codomain_separability}
Let $f: \ZZZ_K^N \times \Theta \to \mathbb{R}$ be a function. Let $f_{\phi}$ be the representation of $f$ in the $\textsc{FPA}(\phi)$ model. For $\theta \in \Theta$, define $\Lambda_{\theta}(f_\phi) := \min_{\sigma_1,\sigma_2 \in \ZZZ_K^N} \abs{f_\phi(\sigma_1,\theta) - f_\phi(\sigma_2,\theta)}$ to be the minimum difference in the function output keeping $\theta$ fixed. For a fixed $\tau > 0$, we say that $f$ admits \emph{$\tau$-codomain separability} using $\theta$ in the $\textsc{FPA}(\phi)$ model if $\Lambda_{\theta}(f_\phi) \geq \tau$.

In particular, we say that $f$ admits \emph{\zero-codomain separability} using $\theta$ in the $\textsc{FPA}(\phi)$ model if there exists any $\tau > 0$ such that $\Lambda_{\theta}(f_\phi) \geq \tau$. 
\end{definition}

Let $f_\phi$ be the representation of $f$ in the $\mathsf{FPA}(\phi)$ model. Informally, representation in the $\mathsf{FPA}(\phi)$ model implies computing $f$ within the granularity defined by $\phi$, and reporting underflow/overflow when the results are out of range. It is easy to show that if $f_\phi$ admits $\tau$-codomain separability using $\theta$, then $f$ also admits $\tau$-codomain separability in the $\mathsf{APA}$ model using $\theta$ (see Proposition~\ref{prop:FPAAPA}). The other direction is trickier, but we can establish that if $f$ admits $\tau$-codomain separability in the $\mathsf{APA}$ model using  $\theta$, then $f_\phi$ admits $\tau$-codomain separability using $\theta$ if $\phi\geq\max(2 \log_2 \tau  + 5,  -2\log_2\tau - 1)$ (see Proposition~\ref{prop:APAFPA}).  

\begin{proposition} \label{prop:FPAAPA}
Let $f:\ZZZ_K^N \times \Theta \to \mathbb{R}$, $\theta \in \Theta$, and $\tau > 0$. Let $f_\phi$ be the representation of $f$ in the $\mathsf{FPA}(\phi)$ model. If $f_\phi$ admits $\tau$-codomain separability in the $\mathsf{FPA}(\phi)$ model (using $\theta$) for some $\phi > 0$, then $f$ also admits $\tau$-codomain separability in the $\mathsf{APA}$ model using $\theta$. Moreover, $f_{\phi'}$ also admits $\tau$-codomain separability in the $\mathsf{FPA}(\phi')$ model using $\theta$ for all $\phi' > \phi$.
\end{proposition}
\begin{proof}
Given $\theta$ and the fact that $f_\phi$ admits $\tau$-codomain separability using $\theta$ in the $\mathsf{FPA}(\phi)$ model, denote
\begin{align*}
    \Lambda_\theta(f_\phi) = \min_{\sigma_1,\sigma_2 \in \ZZZ_K^N} \abs{f_\phi(\sigma_1,\theta) - f_\phi(\sigma_2,\theta)} = b_{\frac{\phi-3}{2}}\cdots b_1 b_0\ .\ b_{-1}b_{-2} \cdots b_{-\frac{\phi-1}{2}} \geq 2\tau, 
\end{align*}
where each bit $b_i \in \{0,1\}$. Then, in the $\mathsf{FPA}(\phi+1)$ model, we can write (without loss of generality):
\begin{align*}
    \Lambda_\theta(f_{\phi+1}) &= \begin{cases}
    0b_{\frac{\phi-3}{2}}\cdots b_1 b_0\ .\ b_{-1}b_{-2} \cdots b_{-\frac{\phi-1}{2}} &\ \ \text{if $\phi$ is even},\\
    b_{\frac{\phi-3}{2}}\cdots b_1 b_0\ .\ b_{-1}b_{-2} \cdots b_{-\frac{\phi-1}{2}}0&\ \ \text{if $\phi$ is odd}
    \end{cases}\\
     &= b_{\frac{\phi-3}{2}}\cdots b_1 b_0\ .\ b_{-1}b_{-2} \cdots b_{-\frac{\phi-1}{2}} \geq 2\tau,
\end{align*}
which establishes that $f_{\phi+1}$ is $\tau$-codomain separable using $\theta$ in the $\mathsf{FPA}(\phi+1)$ model. By induction, this implies that $f_{\phi'}$ admits $\tau$-codomain separation in the $\mathsf{FPA}(\phi')$ model using $\theta$ for all $\phi' > \phi$. In the limit when $\phi \to \infty$, this is equivalent to saying that $f$ admits $\tau$-codomain separation in the $\mathsf{APA}$ model.
\end{proof}

\begin{proposition} \label{prop:APAFPA}
Let $f:\ZZZ_K^N \times \Theta \to \mathbb{R}$, $\theta \in \Theta$, and $\tau > 0$. Let $f_\phi$ be the representation of $f$ in the $\mathsf{FPA}(\phi)$ model.  If $f$ admits $\tau$-codomain separation in the $\mathsf{APA}$ model using some $\theta \in \Theta$ and for some $\tau > 0$, then $f_\phi$ admits $\tau$-codomain separation in the $\mathsf{FPA}(\phi)$ model using $\theta$ for $\phi\geq\max(2 \log_2 \tau  + 5,  -2\log_2\tau - 1)$.
\end{proposition}
\begin{proof}
Suppose $f$ admits $\tau$-codomain separation in the $\mathsf{APA}$ model. To establish $f_\phi$ admits $\tau$-codomain separation in the  $\mathsf{FPA}(\phi)$ model, we need to split the two cases: $\tau>1$ and $\tau<1$.
Represent $\Lambda_{\theta}(f_\phi)$ in the $\mathsf{FPA}(\phi)$ model as 
$$b_{\frac{\phi-3}{2}}\cdots b_1 b_0\ .\ b_{-1}b_{-2} \cdots b_{-\frac{\phi-1}{2}}.$$
When $\tau>1$, a sufficient condition for $f_\phi$ to satisfy the $\tau$-codomain separation is to have enough precision before the decimal point,
$$2^{\frac{\phi-3}{2}}\geq 2\tau,$$
which gives $\frac{\phi-3}{2}  \geq 1 + \log_2 \tau  \Longrightarrow \phi \geq 2 \log_2 \tau  + 5$. When $\tau<1$, a sufficient condition for $f_\phi$ to satisfy $\tau$-codomain separation is to have enough precision after the decimal point,
$$2^{-\frac{\phi-1}{2}}\leq 2\tau,$$
which gives $-\frac{\phi-1}{2}\leq 1 +  \log_2 \tau \Longrightarrow \phi \geq -2\log_2\tau - 1$. The proposition follows from combining the two together to obtain $\phi\geq\max(2 \log_2 \tau  + 5,  -2\log_2\tau - 1)$.
\end{proof}


\section{Missing Details from Section~\ref{sec:linearly_separable}} \label{app:linearly_separable}

We show that all Bregman divergence based loss functions are  linearly-decomposable.  Given a continuously differentiable strictly convex function $F:\mathcal{S} \to \mathbb{R}$ over some closed convex set $\mathcal{S} \subseteq \mathbb{R}^d$, the Bregman divergence $D_F: \mathcal{S} \times \mathcal{S} \to \mathbb{R}$ associated with $F$ is defined as $D_F(\mathbf{p},\mathbf{q}) = F(\mathbf{p})-F(\mathbf{q})-\langle \nabla F(\mathbf{q}), \mathbf{p}-\mathbf{q} \rangle$. We will focus on the binary case for our discussion in this section and assume that the domain of $F$ is the closed convex set $[0,1]^N$.  
 
\textbf{Restatement of Lemma~\ref{lem:bregman_is_linear}}\emph{
Let $F:[0,1] \times [0,1] \to \mathbb{R}$ be a strongly convex function and $D_F(p,q) = F(p) - F(q) - \langle \nabla F(q),p - q\rangle$ be the Bregman divergence associated with $F$. Let $f_F(\sigma,\theta)$ be the corresponding loss function, defined as follows:
$$f_F(\sigma,\theta) = \frac{1}{N}\sum_{i=1}^N D_F\paran{[\sigma_i,1-\sigma_i],[\theta_i,1-\theta_i]}.$$
Then, $f_F(\sigma,\theta)$ is linearly-decomposable.}
\begin{proof}
First, observe that any loss function of the form $f(\theta,\sigma) = \sum_{i=1}^N \paran{\sigma_i g(\theta_i) + (1-\sigma_i)h(\theta_i)}$ is additively linearly separable, since it can be rewritten as $f(\theta,\sigma) = \Sigma_{i=1}^N \sigma_i g'(\theta_i) + \sum_{i=1}^N h(\theta_i)$, where $g'(\theta_i) = g(\theta_i)-h(\theta_i)$. 

Let $F:[0,1] \times [0,1] \to \mathbb{R}$ be a strongly convex function and $DF(p,q) = F(p) - F(q) - \langle \nabla F(q),p - q\rangle$ be the Bregman divergence associated with $F$. Let $\mathcal{L}_F(\theta,\sigma)$ be the corresponding loss function, defined as follows:
$$\mathcal{L}_F(\theta,\sigma) = \frac{1}{N}\sum_{i=1}^N D_F([\sigma_i,1-\sigma_i], [\theta_i,1-\theta_i]).$$

We start by using a shorthand $\phi(x) = F([x,1-x])$. Then, we can write the following:
\begin{align*}
    \mathcal{L}_F(\theta,\sigma) &= \frac{1}{N}\sum_{i=1}^N D_F([\sigma_i,1-\sigma_i], [\theta_i,1-\theta_i])\\
    &= \frac{1}{N} \sum_{i=1}^N \paran{\phi(\sigma_i) - \phi(\theta_i)-\langle \nabla F([\sigma_i, 1-\sigma_i]), [\sigma_i-\theta_i, \theta_i - \sigma_i] \rangle}\\
    &= \frac{1}{N}\sum_{i=1}^N \paran{\phi(\sigma_i) - \phi(\theta_i)- (\sigma_i - \theta_i) \langle \nabla F([\sigma_i, 1-\sigma_i]), [1,-1] \rangle}\\
    &= \frac{1}{N}\paran{\sum_{i:\sigma_i=1} \phi(1) - \phi(\theta_i) - (1-\theta_i)\langle \nabla F([1,0]), [1,-1] \rangle} + \frac{1}{N}\paran{\sum_{i:\sigma_i=0} \phi(0) - \phi(\theta_i) + \theta_i\langle \nabla F([0,1]), [1,-1] \rangle}
\end{align*}
Let $a = \langle \nabla F([1,0]), [1,-1] \rangle$ and $b = \langle \nabla F([0,1]), [1,-1] \rangle$ be constants. Then, we have the following:
\begin{align*}
    \mathcal{L}_F(\theta,\sigma) &= \frac{1}{N}\paran{\sum_{i:\sigma_i=1} \phi(1) - \phi(\theta_i) - a  + \theta_ia} + \frac{1}{N}\paran{\sum_{i:\sigma_i=0} \phi(0) - \phi(\theta_i) + \theta_ib}\\
    &= \frac{1}{N}\sum_{i=1}^N \paran{\sigma_i \paran{\phi(1) - a + \theta_i a} + (1-\sigma_i)\paran{\phi(0) + \theta_i b}}-\frac{1}{N}\sum_{i=1}^N \phi(\theta_i)\\
    &= \sum_{i=1}^N \sigma_ig(\theta_i) + h(\theta),
\end{align*}
where $g(\theta_i) = \frac{1}{N}\paran{\theta_i (a-b)+\phi(1)- \phi(0)-a}$, and $h(\theta) = \phi(0) + \frac{1}{N}\sum_{i=1}^N \paran{\theta_i b - \phi(\theta_i)}$. Thus, a separating vector for the Bregman loss is obtained by setting $g(\theta_i) = \tau \ln p_i$, where $p_i$ is the $i^{th}$ prime number, as follows:
\begin{align*}
    &\frac{1}{N}\paran{\theta_i (a-b)+\phi(1)- \phi(0)-a} = \tau \ln p_i \\
    \implies & \theta_i = \frac{N\tau \ln p_i + \phi(0) + a - \phi(1)}{a-b} \\
    \implies & \theta_i = \frac{N\tau \ln p_i + F([0,1]) + \langle \nabla F([1,0]), [1,-1] \rangle - F([1,0])}{\langle \nabla F([1,0]), [1,-1] \rangle-\langle \nabla F([0,1]), [1,-1] \rangle}. \qedhere
\end{align*}
\end{proof}

\noindent\textbf{Restatement of Theorem~\ref{thm:linear_separable}.} \emph{Let $g:[0,1] \to \mathbb{R}$ be some deterministic function and $f:\{0,1\}^N \times (0,1)^N \to \mathbb{R}$ be a loss function that is $g$-linearly-decomposable (Definition~\ref{def:linear_separable}). Then, for any $\tau > 0$, the function $f$ is $2\tau$-codomain separable if there exists $\theta \in (0,1)^N$ so that $g(\theta_i) - g(1-\theta_i) > 2^iN\tau$ for all $i \in [N]$. If $\tau=0$, then setting $g(\theta_i) - g(1-\theta_i) > 0$ for all $i \in [N]$ suffices for \zero-codomain separability.}
\begin{proof}
We begin by observing that~\eqref{eq:linear_separable} can be rewritten as follows:
\begin{align*}
    f(\sigma,\theta) &= \frac{1}{N}\sum_{i=1}^N\paran{\sigma_i g(\theta_i) + (1-\sigma_i)g(1-\theta_i)} \\ &= \frac{1}{N} \paran{\sum_{i:\sigma(i)=1}\paran{g(\theta_i)-g(1-\theta_i)}+\sum_{i=1}^N g(1-\theta_i)}.
\end{align*}
For any $\theta\in(0,1)^N$, we can then write the following:
\begin{align*}
    \Lambda_\theta(f) &= \min_{\sigma_1,\sigma_2 \in \{0,1\}^N} \abs{f(\sigma_1,\theta)-f(\sigma_2,\theta)}\\
    &= \frac{1}{N} \min_{\sigma_1,\sigma_2 \in \{0,1\}^N} \abs{\sum_{i:\sigma_1(i)=1}\paran{g(\theta_i)-g(1-\theta_i)} - \sum_{j:\sigma_2(j)=1}\paran{g(\theta_j)-g(1-\theta_j)}}
\end{align*}
If for all $i \in [N]$, it holds that $g(\theta_i) - g(1-\theta_i) = 2^iN\tau(1+\delta)$ for some $\delta > 0$, then:
\begin{align*}
    \Lambda_\theta(f) &= (2\tau(1+\delta)) \min_{\sigma_1,\sigma_2 \in \{0,1\}^N} \abs{\sum_{i:\sigma_1(i)=1} 2^{i-1} + \sum_{j:\sigma_1(j)=1}2^{j-1}} = 2\tau(1+\delta) > 2\tau,
\end{align*}
where the last step holds because $\sigma_1 \neq \sigma_2$. 
\end{proof}

\subsection{Kullback-Leibler Divergence Loss} 
\label{sec:kl_loss}
The (generalized) Kullback-Leibler (KL) divergence between vectors $\mathbf{p}, \mathbf{q} \in \mathcal{S} \subseteq \mathbb{R}^d$ is defined as:
$$D_{\text{KL}}(\mathbf{p},\mathbf{q}) = \sum_{i \in [d]} p_i \ln\frac{p_i}{q_i} - \sum_{i \in [d]} (p_i-q_i), $$
where $\mathbf{p}=(p_1,\dots,p_d)$ and $\mathbf{q}=(q_1,\dots,q_d)$.

For a binary classification setting, considering the $i$th datapoint, we have the true label $\sigma_i \in \{0,1\}$ and $\theta_i \in (0,1)$ which is the probability assigned to the event $\sigma_i =1$ by the ML model.  In that case, we have
$$D_{\text{KL}}([\sigma_i,1-\sigma_i],[\theta_i,1-\theta_i]) = - \sigma_i \ln \theta_i - (1-\sigma_i) \ln(1-\theta_i).
$$
Summing over the $N$ datapoints (and dividing by $N$) gives the Kullback-Leibler divergence loss,
 \begin{align} \label{eq:KL}
    \textsc{KLLoss}(\sigma,\theta) & = \frac{1}{N}\sum_{i=1}^N  D_{\text{KL}}([\sigma_i,1-\sigma_i],[\theta_i,1-\theta_i]) = \frac{-1}{N}  \sum_{i=1}^N \sigma_i \ln \theta_i + (1-\sigma_i) \ln(1-\theta_i) \nonumber \\
    & = \frac{-1}{N} \paran{ \sum_{i: \sigma_i = 1} \ln \theta_i + \sum_{i: \sigma_i = 0} \ln (1-\theta_i)},
\end{align}
which is exactly the binary cross-entropy loss\footnote{Here, we adopt the notion that $0\ln 0=0$, so that KL divergence is well-defined.}. 

\subsection{Itakura-Saito Divergence Loss} 
\label{app:IS_loss}
The Itakura-Saito divergence for vectors $\mathbf{p}, \mathbf{q} \in \mathcal{S} \subseteq \mathbb{R}^d$ is defined as: $$D_{\text{IS}}(\mathbf{p},\mathbf{q}) = \sum_{i \in [d]} \paran{\frac{p_i}{q_i} - \ln \frac{p_i}{q_i} - 1},$$
where $\mathbf{p}=(p_1,\dots,p_d)$ and $\mathbf{q}=(q_1,\dots,q_d)$.

For a binary classification setting, considering the $i$th datapoint, we have the true label $\sigma_i \in \{0,1\}$ and $\theta_i \in (0,1)$, which is the probability assigned to the event $\sigma_i = 1$ by the ML model. In this case, based on $D_{\text{IS}}$, the Itakura-Saito divergence loss is defined as:
\begin{align} 
\label{eq:IS}
  \textsc{ISLoss}(\sigma,\theta) =  \frac{1}{N}  \paran{\sum_{i:\sigma_i=1} \paran{\frac{1}{\theta_i}+\ln \theta_i - 1} + \sum_{i:\sigma_i=0}  \paran{\frac{1}{1-\theta_i}+\ln (1-\theta_i) - 1}}  
\end{align}
The above equation shows the linear decomposability of this loss, therefore, Theorem~\ref{thm:linear_separable} can be applied to get the following result.

\noindent\textbf{Restatement of Corollary~\ref{cor:IS}.} \emph{The Itakura-Saito divergence loss (\textsc{ISLoss}) is $2\tau$-codomain separable with $\theta_i = \paran{1+3^{2^iN\tau}}^{-1}$.}
\begin{proof}
We apply Theorem~\ref{thm:linear_separable} here. For the Itakura-Saito divergence loss in~\eqref{eq:IS}, we begin by noticing that for $x\in\paran{0,1/2}$, it holds that $$\frac{1}{x}-\frac{1}{1-x}+\ln\frac{x}{1-x} > \ln\frac{1-x}{x} > 0.$$ 
Thus, since 
$$g(\theta_i) - g(1-\theta_i) = \frac{1}{\theta_i}-\frac{1}{1-\theta_i}+\ln\parfrac{\theta_i}{1-\theta_i}$$ 
for this loss, it suffices to ensure that $\ln\parfrac{1-\theta_i}{\theta_i} > 2^iN\tau$ for Theorem~\ref{thm:linear_separable} to apply. In particular, we solve $\ln\parfrac{1-\theta_i}{\theta_i} = 2^iN\tau\ln 3$ to obtain $\theta_i = \paran{1+3^{2^iN\tau}}^{-1}$. Note that $\theta_i < 1/2$ as needed above.
\end{proof}

\subsection{Squared Euclidean Loss} 
\label{app:sq_euc_loss}
The squared Euclidean divergence for vectors $\mathbf{p}, \mathbf{q} \in \mathcal{S} \subseteq \mathbb{R}^d$ is defined as:
\begin{align}
    \label{eq:sqEuc}
    D_{\text{SE}}(\mathbf{p},\mathbf{q}) = \sum_{i \in [d]} \paran{\abs{p_i - q_i}^2},
\end{align}
where $\mathbf{p}=(p_1,\dots,p_d)$ and $\mathbf{q}=(q_1,\dots,q_d)$.

Again for the binary classification setting, considering the $i$th datapoint, we get the following expression for this loss:
$$D_{\text{SE}}([\sigma_i,1-\sigma_i],[\theta_i,1-\theta_i])= 2\| \sigma_i -\theta_i \|^2.$$ 
Summing over the $N$ datapoints (and dividing by $N$, and ignoring the factor of $2$), we get the squared Euclidean loss as follows:
\begin{align} 
\label{eq:sq}
\textsc{SELoss} (\sigma,\theta) &=  \frac{1}{N}  \paran{\sum_{i:\sigma_i=1} \abs{\sigma_i - \theta_i}^2 + \sum_{i:\sigma_i=0} \abs{\sigma_i - \theta_i}^2} = \frac{1}{N}  \paran{\sum_{i:\sigma_i=1} ( 1 - \theta_i)^2 + \sum_{i:\sigma_i=0} \theta_i^2}.
\end{align}

In this case, we establish \zero-codomain separability. 

\noindent\textbf{Restatement of the first part of Corollary~\ref{cor:eucl}.} \emph{The squared Euclidean loss (\textsc{SELoss})  is  \zero-codomain separable using $\theta_i = (1/2)\paran{1-\ln(p_i)/N}$, where $p_i$ is the $i$th prime number.}
\begin{proof}
To apply Theorem~\ref{thm:linear_separable} to the squared Euclidean loss, we have $g(\theta_i) = (1-\theta_i)^2$, which gives $g(\theta_i) - g(1-\theta_i) = 1-2\theta_i$. Setting this to $\frac{\ln p_i}{N}$, where $p_i$ is the $i$th prime number, ensures that $\mu(S_\theta)>0$. Equivalently, $\theta_i = \frac{1}{2}\paran{1-\frac{\ln p_i}{N}}$ works.
\end{proof}

Note that the proof above assumes that $\theta \in (0,1)^N$. We show in Theorem~\ref{thm:lp_loss} that restricting $\theta$ to $\{0,1\}^N$ prohibits $\tau$-codomain separability for any $\tau > 0$.

\subsubsection{Norm-like Divergence Loss} 
\label{app:norm_like_loss}
The norm-like divergence for vectors $\mathbf{p}, \mathbf{q} \in \mathcal{S} \subseteq \mathbb{R}^d$ and $\alpha \geq 2$ is defined as:
$$D_{\text{NL}}(\mathbf{p},\mathbf{q}) = \sum_{i \in [d]} \paran{p_i^\alpha + (\alpha-1)q_i^\alpha - \alpha p_i q_i^{\alpha-1}}, $$
where $\mathbf{p}=(p_1,\dots,p_d)$ and $\mathbf{q}=(q_1,\dots,q_d)$.
Again for binary classification, considering the $i$th datapoint, we get the following expression for this loss:
\begin{multline*}
D_{\text{NL}}([\sigma_i,1-\sigma_i],[\theta_i,1-\theta_i]) 
\\= \paran{ \sigma_i^\alpha + (\alpha-1) \theta_i^\alpha - \alpha \sigma_i \theta_i^{\alpha-1} + (1-\sigma_i)^\alpha + (\alpha-1) (1-\theta_i)^\alpha - \alpha (1-\sigma_i) (1-\theta_i)^{\alpha-1}}.
\end{multline*}
Summing over the $N$ datapoints (and dividing by $N$), and simplifying gives the norm-like divergence loss.
\begin{multline}
\label{eq:norm}
\textsc{NLLoss}(\sigma,\theta)= \frac{1}{N}  \big (\sum_{i:\sigma_i=1} \paran{1 + (\alpha-1) \theta_i^\alpha - \alpha \theta_i^{\alpha-1} + (\alpha-1)(1-\theta_i)^\alpha} 
\\ + \sum_{i:\sigma_i=0} \paran{1 + (\alpha-1) (1-\theta_i)^\alpha - \alpha(1-\theta_i)^{\alpha-1} + (\alpha-1) \theta_i^\alpha}
 \big ).
\end{multline}
In this case, we establish \zero-codomain separability. 

\noindent\textbf{Restatement of the second part of Corollary~\ref{cor:eucl}.} \emph{The norm-like divergence loss (\textsc{NLLoss}) for $\alpha \geq 2$ is \zero-codomain separable using $\theta$ where: $(1-\theta_i)^{\alpha-1}-\theta_i^{\alpha-1} = (\ln p_i)/(N\alpha)$ with $p_i$ as the $i$th prime number.}
\begin{proof}
Here we have $g(\theta_i) = 1+(\alpha-1)\theta_i^\alpha - \alpha\theta_i^{\alpha-1}+(\alpha-1)(1-\theta_i)^\alpha$, which gives $g(\theta_i) - g(1-\theta_i) = \alpha((1-\theta_i)^{\alpha-1} - \theta_i^{\alpha-1})$. Similar to above, setting $g(\theta_i) - g(1-\theta_i) = \frac{\ln p_i}{N}$ suffices. This is equivalent to finding a solution to the following equation, which has a unique solution in $(0,1)$ for any fixed $\alpha \geq 2$: $$(1-\theta_i)^{\alpha-1}-\theta_i^{\alpha-1} = \frac{\ln p_i}{N\alpha}.$$It is easy to see that such a $\theta_i < 1/2$ exists.
\end{proof}




\begin{align} \label{eq:sqfree_min_separation}
    \Lambda_\theta(f) = \min_{\sigma_1 \ne \sigma_2}\abs{\sum_{i:\sigma_1(i)=1} g(\theta_i) - \sum_{i:\sigma_2(i)=1} g(\theta_i)}.
\end{align}

\textbf{Restatement of Theorem~\ref{thm:linearly_decomposable_using_primes}} \emph{
Let $f:\{0,1\}^N \times (0,1)^N \to \mathbb{R}$ be a loss function that is linearly-decomposable (Definition~\ref{def:sqfree_linearly_separable}). Let $p_i$ is the $i^{th}$ prime number and $P=\prod_{i=1}^N p_i$, is the product of the first $N$ primes. Then, for any $\tau > 0$, setting $g(\theta_i)=3P\tau  \ln p_i$ for loss functions in Equation~\ref{def:sqfree_linearly_separable} ensures that $\Lambda_\theta(f) \geq 2\tau$. If $\tau=0$, setting $g(\theta_i)=\ln p_i$ suffices for \zero-codomain separability.}
\begin{proof}
We prove this by substituting $g(\theta_i) = 6P\tau \ln p_i$ in Equation~\ref{eq:sqfree_min_separation} as follows:
\begin{align*}
    \Lambda_\theta(f) &= \min_{\sigma_1 \ne \sigma_2}\abs{\sum_{i:\sigma_1(i)=1} g(\theta_i) - \sum_{j:\sigma_2(j)=1} g(\theta_j)}\\
    &= (3P\tau)\ \min_{\sigma_1 \ne \sigma_2}\abs{\sum_{i:\sigma_1(i)=1} \ln p_i - \sum_{j:\sigma_2(j)=1} \ln p_j}\\
    &= (3P\tau)\ \min_{\substack{S_1,S_2 \subseteq [N]\\S_1 \cap S_2 = \emptyset}} \abs{\ln \parfrac{\prod_{i \in S_1}p_i}{\prod_{s \in S_2}p_j}} \geq 2\tau.
\end{align*}
Too see why the last inequality holds, let $S_1^*$ and $S_2^*$ denote the sets of primes that achieve the minimum value above. Then, without loss of generality, it must hold that $\prod_{j \in S_2^*}p_j \leq P$ and $\prod_{j \in S_2^*}p_j \leq 1 + \prod_{i \in S_1^*}p_i$. Thus, using the fact that $1.5x\ln \parfrac{1+x}{x} > 1$ for all $x \geq 1$, we obtain that $1.5P$ times the log expression above is at least $1$. Further scaling by $2\tau$ gives $\Lambda_\theta(f) \geq 2\tau$.
\end{proof}

\section{Missing Details from Section~\ref{sec:logloss}} \label{app:logloss}
\noindent\textbf{Worked out example for \zero-codomain separability for multiclass cross-entropy loss.} For illustration, we provide a simple example to demonstrate \zero-codomain separability for the multiclass cross-entropy loss using a construction of prediction vector from~\citep{aggarwal2021icml}. Note that in Theorem~\ref{thm:multi_class_proof}, we establish that in fact, multiclass cross-entropy loss admits the stronger notion of $\tau$-codomain separability for any $\tau > 0$.

Assume $N=2$ and $K = 3$. Construct a matrix $\theta$ with first row  $[\frac{2}{10},\frac{3}{10},\frac{5}{10}]$ and second row $[\frac{7}{31},\frac{11}{31},\frac{13}{31}]$. 
Observe that these vectors are chosen using unique prime numbers in the numerator (the denominator is for normalizing the sum to 1), the reasoning for which will be clear shortly. Using $\theta$, one can prove that the cross-entropy loss will be distinct for every labeling by observing that the terms inside the logarithm, that are chosen for the outer sum in~\eqref{eq:multi_class_log_loss}, are distinct for all labelings. For example, if the true labeling is $[0,2]$, then we obtain $\Klogloss([0,2]; \theta) = -\frac{1}{2}\paran{\ln \frac{2}{10} + \ln \frac{13}{31}} = -\frac{1}{2} \ln\paran{\frac{2 \cdot 13}{10 \cdot 31}}$. Similarly, if the true labeling is $[1,0]$, then we obtain $\Klogloss([1,0]; \theta) = -\frac{1}{2}\paran{\ln \frac{3}{10} + \ln \frac{7}{31}} = -\frac{1}{2}\ln\paran{\frac{3 \cdot 7}{10 \cdot 31}} $. The use of primes makes this selection of summands in the $\Klogloss$ score uniquely defined by the true labeling. This follows as the only thing that changes in the $\Klogloss$ score based on the true labeling is the numerator in the $\ln$ term, which is a product of primes based on true labeling. Since the product of primes has a unique factorization, we can recover which primes were used from the product, and since each entry in the matrix $\theta$ is associated with a unique prime, this recovers the true labels. 

\noindent\textbf{Missing proofs.} We show that the  (multiclass) $K$-ary cross-entropy loss is $\tau$-codomain separable for any $\tau > 0$.

\noindent\textbf{Restatement of Theorem~\ref{thm:multi_class_proof}.} \emph{Let $\tau > 0$. Define matrices $\vartheta, \theta \in \mathbb{R}^{N \times K}$ such that 
$$\vartheta_{n,k} = 3^{\paran{2^{(n-1)K+k}N\tau}} \mbox{ and } \theta_{n,k} = \vartheta_{n,k} / \sum_{k=1}^K \vartheta_{n,k}.$$ Then, it holds that $\Klogloss$ is $2 \tau$-codomain separable using $\theta$. If $\tau = 0$, then using $\vartheta_{n,k} = 3^{\paran{2^{(n-1)K+k}}}$ ensures \zero-codomain separability.}
\begin{proof}
We begin by simplifying the expression for $\Klogloss\paran{\theta,\sigma}$ to write it as a sum of two terms: one dependent on the labeling $\sigma$, and the other independent of this labeling. 
\begin{align}\label{eq:multi_class_log_loss_difference}
    \Klogloss\paran{\theta,\sigma} &= \frac{-1}{N}\sum_{i=1}^N \sum_{k=1}^K \ \Big([\sigma_i = k] \cdot \ln \theta_{i,k} \Big) 
    = \frac{-1}{N}\paran{\underbrace{\sum_{i=1}^N \ln \vartheta_{i,\sigma_i}}_{\substack{\text{Labeling Dependent}\\\text{Term}}} - \underbrace{\sum_{i=1}^N \ln \paran{\sum_{k=1}^K \vartheta_{i,k}}}_{\substack{\text{Labeling Independent}\\\text{Term}}}}.
\end{align}
Using~\eqref{eq:multi_class_log_loss_difference}, we then obtain the following:
\begin{align*}
    \Lambda_{\theta}(\Klogloss) &= \min_{\sigma_1, \sigma_2 \in \ZZZ_K^N} \abs{\Klogloss\paran{\theta,\sigma_1} - \Klogloss\paran{\theta,\sigma_2}}\\
    &= \min_{\sigma_1, \sigma_2 \in \ZZZ_K^N} \frac{1}{N}\abs{\sum_{i=1}^N \ \Big(\ln \theta_{i,\sigma_1(i)} \Big) - \sum_{i=1}^N \ \Big(\ln \theta_{i,\sigma_2(i)} \Big)}\\
    &= \min_{\sigma_1, \sigma_2 \in \ZZZ_K^N} \frac{1}{N}\abs{\sum_{i=1}^N \ \Big(\ln \vartheta_{i,\sigma_1(i)} \Big) - \sum_{i=1}^N \ \Big(\ln \vartheta_{i,\sigma_2(i)} \Big)}\\
    &= \min_{\sigma_1, \sigma_2 \in \ZZZ_K^N} (\tau\ln 3) \abs{\sum_{i=1}^N 2^{(i-1)K+\sigma_1(i)} - \sum_{i=1}^N 2^{(i-1)K+\sigma_2(i)}} \geq 2\tau\ln 3 > 2\tau.
\end{align*}
\end{proof}

\subsection{Sigmoid Cross-Entropy Loss} \label{app:sigmoid}

The separability from Theorem~\ref{thm:multi_class_proof} also holds if we apply any bijective activation function before applying the cross-entropy loss. As an example, consider the sigmoid cross-entropy commonly used in the binary classification setting (to compresses arbitrary reals into the range $(0,1)$), defined as follows for $\sigma \in \{0,1\}^N$, $\theta \in (0,1)^N$:
\begin{align}
\label{eq:sigmoid_definition}
	 \frac{-1}{N}\sum_{i=1}^N \ \Big( \sigma_i \ln \paran{\textsc{Sigmoid}(\theta_i)} + (1-\sigma_i)\ln \paran{1-\textsc{Sigmoid}(\theta_i)} \Big),
\end{align}
where $\textsc{Sigmoid}(x) = \paran{1+e^{-x}}^{-1}$ is the sigmoid function. Since $\textsc{Sigmoid}:\mathbb{R} \to (0,1)$ is a bijection (and hence, invertible), given $\textsc{Sigmoid}(x) = y$, we can obtain $x = \ln (y/(1-y))$. Thus, given the matrix $\theta \in [0,1]^{N \times 2}$ from Theorem~\ref{thm:multi_class_proof}, we can construct $\theta' \in (0,1)^{N}$ such that $\theta'_{i} = \ln (\theta_{i,1}/(1-\theta_{i,1}))$ for all $i \in [N]$.
 Once $\theta'$ is obtained, since $\textsc{Sigmoid}(\theta_{i}) = \theta_{i,1}$ and $1-\textsc{Sigmoid}(\theta_{i}) = \theta_{i,2}$  we get that sigmoid cross-entropy loss is $2\tau $-codomain separable using $\theta'$, and hence the approach outlined in \LabelInf~\eqref{eq:labelinf} can be used for $\tau$-robust label inference.

\section{Missing Details from Section~\ref{sec:neural_network}} \label{app:neural_network}

\begin{theorem}\label{thm:nn} 
Let $\tau > 0$ and let $\theta \in (0,1)^{d_2}$ be such that $f:\ZZZ_K \times (0,1)^{d_2} \to \mathbb{R}$ is $2\tau$-codomain separable using $\theta$. Then, for any input $\mathbf{v} \in (0,1)^{d_2}$, given $\ell$ such that $\abs{f(\sigma_\mathbf{v},\textsc{MutNet}_\theta(\mathbf{v})) - \ell} \leq \tau$, the approach outlined in~\LabelInf~\eqref{eq:labelinf} recovers $\sigma_\mathbf{v}$.
\end{theorem}
\begin{proof}
It suffices to show that for given $\theta \in (0,1)^{d_2}$ and any $\mathbf{v} \in (0,1)^{d_2}$, the construction above ensures that $\textsc{MutNet}_\theta(\mathbf{v}) = \theta$. To see this, observe that since all entries in $M_1$ are negative, the product $\mathbf{v}^\top M_1$ has non-positive entries. Thus, when $\textsc{ReLU}$ is applied to $\mathbf{v}^\top M_1$ (element-wise), the output is the zero vector. This zero vector, when fed into the Sigmoid, produces the desired output $\theta$ since $\mathbf{x}'$ is constructed in a way such that $\textsc{Sigmoid}(\mathbf{x}') = \theta$ (element -wise).
\end{proof}

\section{Some Negative Results on $\tau$-codomain Separability}
\label{sec:negative_results}
We now show certain loss functions are not $\tau$-codomain separable. This complements our positive results on $\tau$-codomain separability for cross-entropy and its variants, and Bregman divergence based losses. These negative results on codomain separability rules out label inference in these cases, because of the connections between these two notion established in Proposition~\ref{prop:bound_on_tau_for_label_inference}.

\noindent\textbf{Discrete $L_p$-losses.}
We start with the simple $L_p$-loss defined on the {\em discrete} domain and show it is not $\tau$-codomain separable for any $\tau > 0$.
\begin{theorem}
\label{thm:lp_loss}
For any $p > 0$, the function $f: \{0,1\}^N \times \{0,1\}^N \to \mathbb{R}$ of the form $f(\sigma,\theta) = \pnorm{p}{\theta-\sigma}$ is not $\tau$-codomain separable for any $\tau > 0$.
\end{theorem}
\begin{proof}
Fix some $\theta \in \{0,1\}^N$. For any $\sigma \in \{0,1\}^N$, let $I(\sigma,\theta) = \{i \in [N] \ | \ \sigma(i) \neq \theta(i)\}$ be the set of indices on which $\sigma$ and $\theta$ differ. Then, we can simplify the expression for $f$ as follows:
\begin{align*}
    f(\sigma,\theta) = \paran{\sum_{i=1}^N \abs{\theta(i)-\sigma(i)}^p}^{1/p} = \paran{\sum_{i \in I(\sigma,\theta)} \abs{\theta(i)-\sigma(i)}^p}^{1/p} = \abs{I(\sigma,\theta)}^{1/p}.
\end{align*}
Now, let $\sigma_1, \sigma_2 \in \{0,1\}^N$ be such that they differ from $\theta$ in exactly one label, i.e., $|I(\sigma_1,\theta)| = |I(\sigma_2,\theta)| = 1$ and hence, $f(\sigma_1,\theta) = f(\sigma_2,\theta) = 1$. Note that for any choice of $\theta$, there are $N-1$ such labelings. Thus, $\Lambda_\theta(f) = 0$.
\end{proof}

\noindent\textbf{Set-valued Functions.} We now study set-valued loss functions. These are functions that are expressed with respect to a fixed set, as a mapping from subsets of this set to the real line. For example, in our context of codomain separability (in the binary classification setting), the set of interest is that of the $N$ datapoints, and the subsets are interpreted as comprising of those that have been assigned label $1$. For example, if $N=3$ and the subset is $\{1,3\}$, then this would represent the case where datapoints $1$ and $3$ have labels $1$, and datapoint $2$ has label $0$. As we will see, this generalization helps compute upper bounds on the magnitude of noise that will admit label inference (in a single query) using any prediction vector.

We now present our main results in this section. For the discussion here, we will assume $\Omega = \{s_1,\dots,s_N\}$ to denote a set and $2^\Omega$ to denote the power set of $\Omega$. As mentioned before, since the sets of interest in our application can be thought of as the labels for the datapoints, we will assume $|\Omega| = N$, unless mentioned otherwise.

\begin{theorem}
\label{thm:separability_monotonic}
Let $\Omega = \{s_1,\dots,s_N\}$ be a set. Let $f: 2^\Omega \times \Theta \to \mathbb{R}_+$ be a function and $\theta \in \mathbb{R}^N$ be such that $f(\cdot,\theta)$ is monotonic, i.e.,, for all $A \subseteq B \subseteq [N]$, it holds that $f(A,\theta) \leq f(B,\theta)$. Then, $f$ is not $\tau$-codomain separable using $\theta$ for any 
$$\tau > \min_{B \subset [N]}\min_{j \not\in B} \parfrac{f(B \cup \{j\},\theta) - f(B,\theta)}{2}.$$
In particular, if $f(\emptyset,\theta) = 0$, then $f$ is not $\tau$-codomain separable using $\theta$ for any $\tau > \frac{1}{2} \paran{\min_{j \in [N]} f(\{j\},\theta)}$.
\end{theorem}
\begin{proof}
Fix some $\sigma \in [0,1]^N$. Then, for $f$ to be $\tau$-codomain separable using $\sigma$, it must hold that:
\begin{align*}
    &\forall B\subset [N], j \not\in B.\ \ \ \ \ \abs{f(B \cup \{j\},\theta) - f(B,\theta)} \ge 2\tau \nonumber\\
    \Longrightarrow &\forall B\subset [N].\ \ \ \ \ \tau \le \frac{1}{2} \paran{\min_{j \not\in B} f(B \cup \{j\},\theta) - f(B,\theta)} \ \ \ \text{(since $f(\theta,\cdot)$ is monotonic)}\nonumber\\
    \Longleftrightarrow &\ \tau \le \min_{B \subset [N]}\min_{j \not\in B} \parfrac{f(B \cup \{j\},\theta) - f(B,\theta)}{2}.
\end{align*}
Taking the contrapositive of this statement establishes the desired result. When $f(\emptyset,\theta)=0$, then setting $B = \emptyset$ gives the desired result.
\end{proof}

\begin{corollary}
Let $f: 2^\Omega \times \Theta  \to \mathbb{R}_+$ be a function such that $f(\cdot,\theta)$ is monotonic for all $\theta \in \mathbb{R}^N$. Then, $f$ is not $\tau$-codomain separable for any $$\tau > \sup_{\theta \in \mathbb{R}^N}\min_{B \subset [N]}\min_{j \not\in B} \parfrac{f(B \cup \{j\},\theta) - f(B,\theta)}{2}.$$In particular, if $f(\emptyset,\theta) = 0$ for all $\theta \in \mathbb{R}^N$, then $f$ is not $\tau$-codomain separable for any $$\tau > \frac{1}{2} \paran{\sup_{\theta \in \mathbb{R}^N}\min_{j \in [N]} f(\{j\},\theta)}.$$
\end{corollary}

We now show that if in addition to monotonicity, the loss function is also bounded, then we can get stronger negative results.  

\begin{theorem}
\label{thm:separability_monotonic_bounded}
Let $\Omega = \{s_1,\dots,s_N\}$ be a set. Let $f: 2^\Omega \times \Theta  \to \mathbb{R}_+$ be a function such that $f(\cdot,\theta)$ is monotonic and $f(\cdot,\theta) \leq \beta$ for all $\theta \in \mathbb{R}^N$ and for some finite $\beta > 0$. Then, $f$ is not $\tau$-codomain separable for any $\tau > \beta/N$.
\end{theorem}
\begin{proof}
Assume that $f$ is $\tau$-codomain separable using some $\theta$. Consider the chain of values $v_0 = f(\{\ \},\theta), v_1 = f(\{1\},\theta), f(\{1,2\},\theta),\dots,v_N = f(\{1,\dots,N\},\theta)$. For each $i \in [N]$, since $f$ is $\tau$-codomain separable, we must have $|v_i - v_{i-1}| \geq \tau$. Since $f$ is monotonic, this implies $v_i - v_{i-1} \geq \tau$. Summing both sides over $i$ gives $\sum_{i=1}^N \paran{v_i - v_{i-1}} = v_N - v_0$, which must be at least $N\tau$ for the inequality above to hold. This implies $v_N \geq N\tau + v_0$, which, for $\tau > \beta/N$ gives $v_N > \beta$ (since $f$ is non-negative). This is a contradiction since $f$ is bounded above by $\beta$.
\end{proof}

%% file: supplement_experiments.tex
\section{Missing Details from Section~\ref{sec:experiments}}
\label{app:experiments}

\captionsetup[subfloat]{labelformat=empty}
\begin{figure*}[t]
    \centering
    \vspace{-0.5em}
    \subfloat[]{\includegraphics[width=150pt]{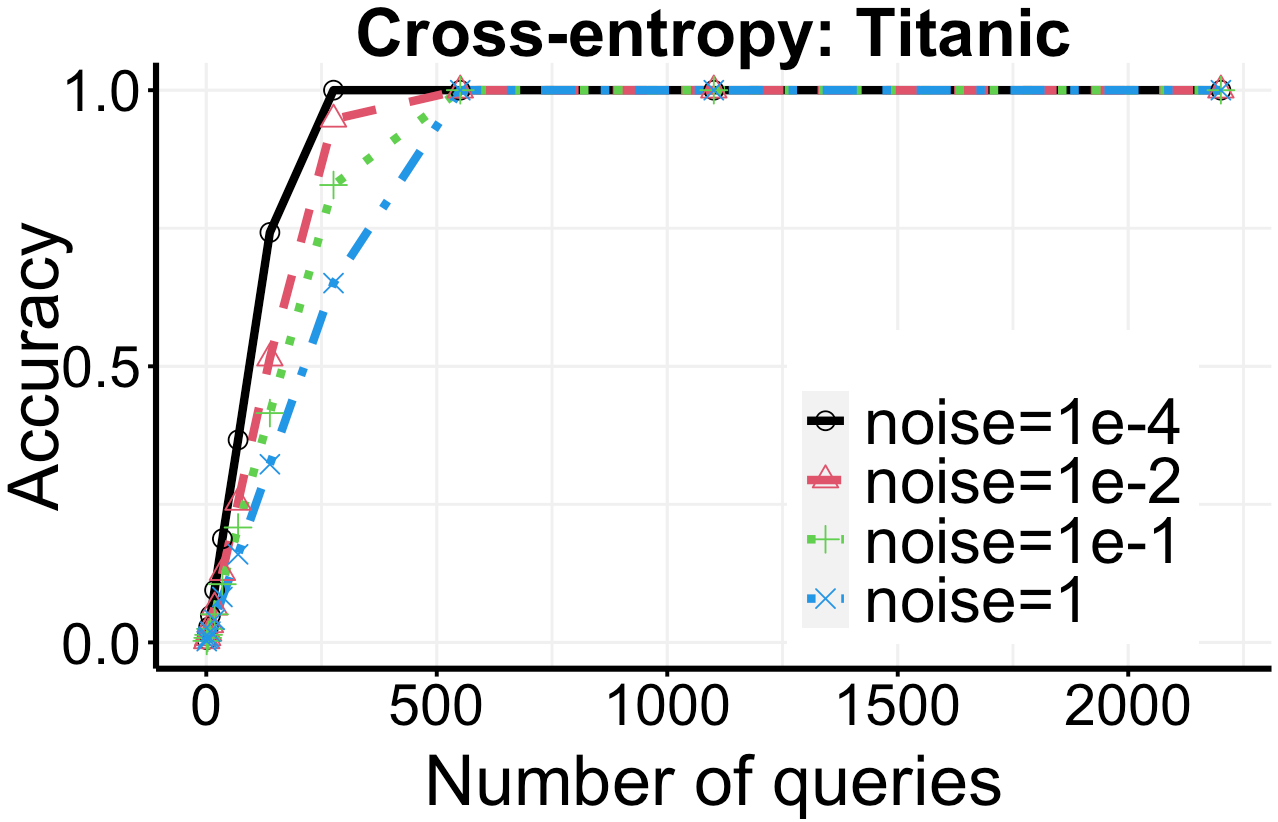}}
    \hspace{0.1em}
    \hspace{0.1em}
    \subfloat[]{\includegraphics[width=150pt]{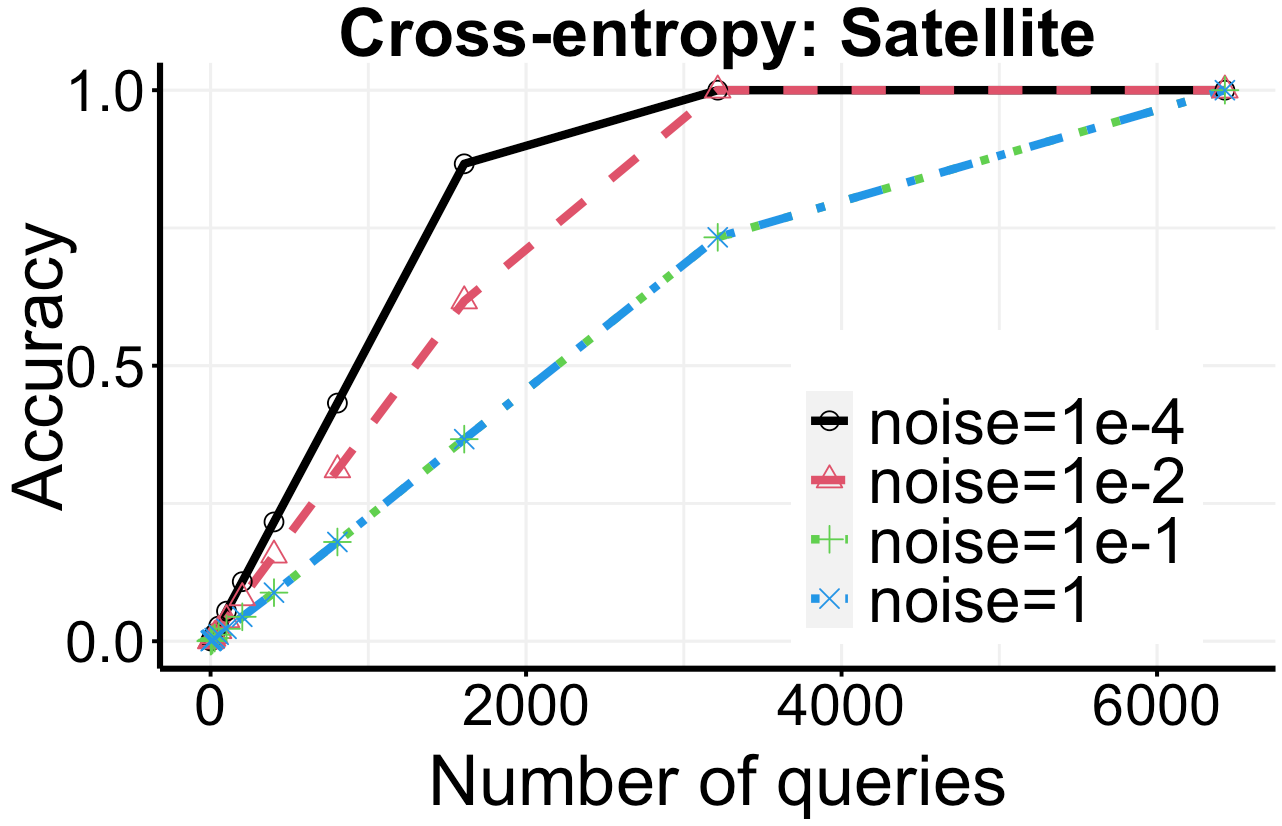}} 
    \\
    \subfloat[]{\includegraphics[width=150pt]{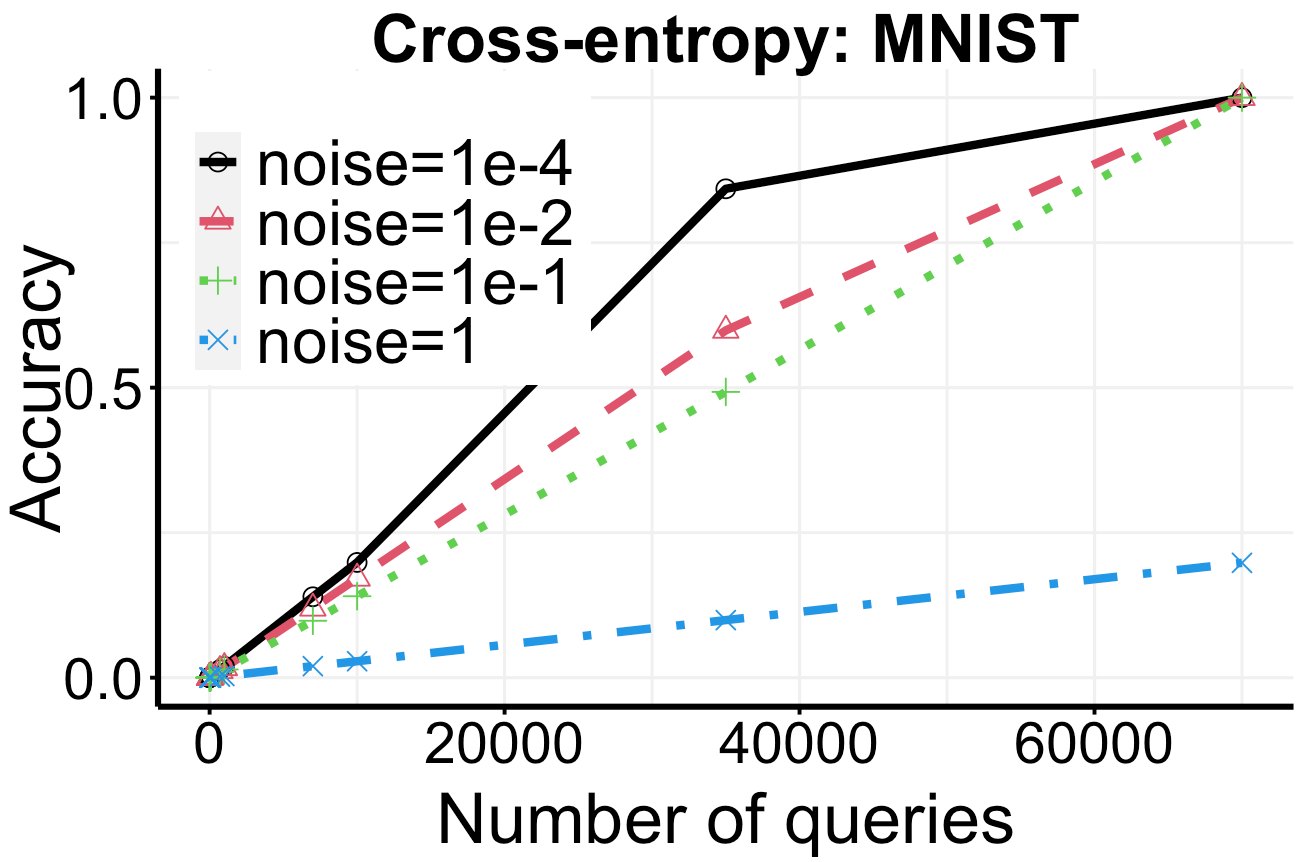}} 
    \hspace{0.1em}
    \subfloat[]{\includegraphics[width=150pt]{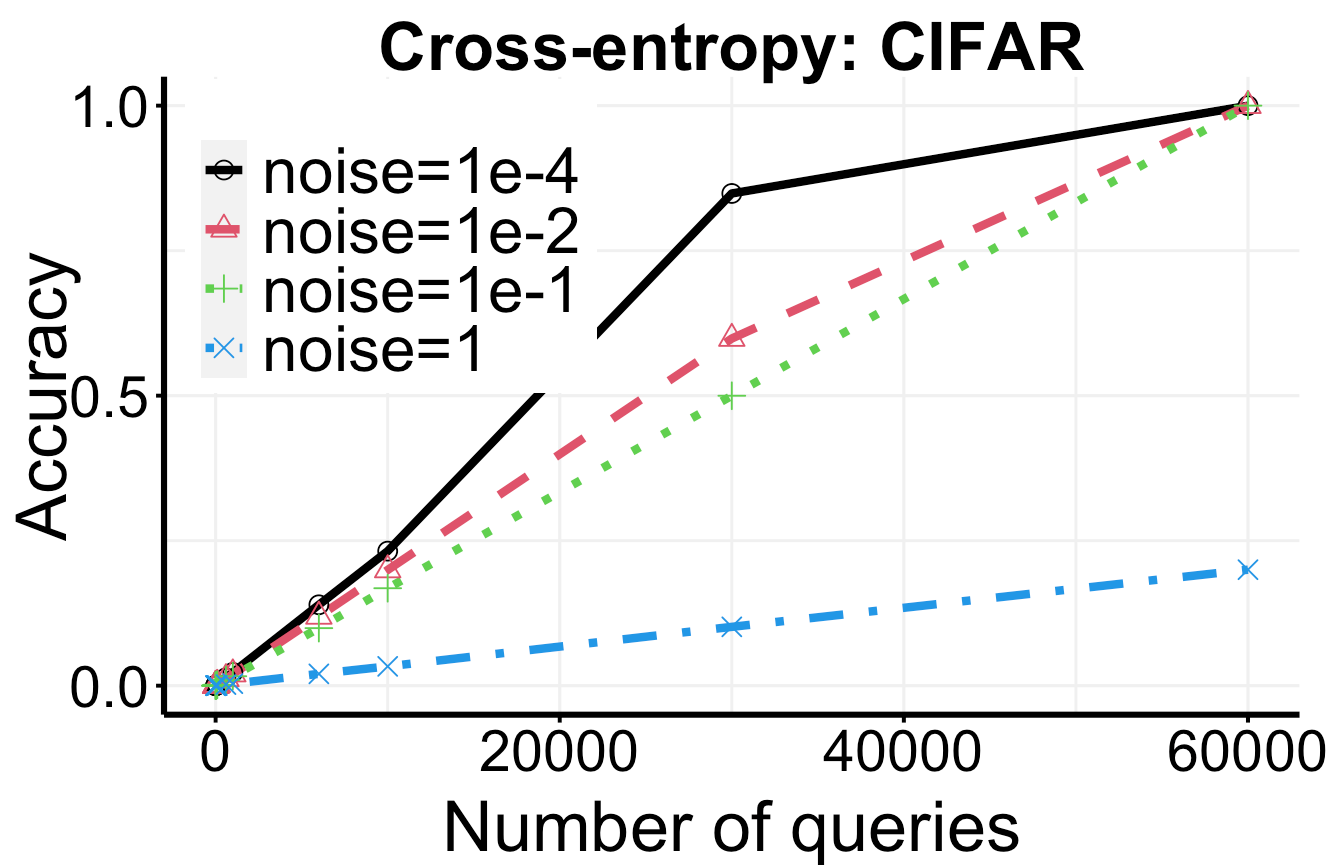}}
    \vspace*{-3ex}
    \caption{Label reconstruction accuracy with the multi-query label inference attack.}
    \label{fig:logloss}
    \vspace{2ex}
\end{figure*} 

\begin{figure}[t]
    \centering
    \vspace{-0.5em}
    \subfloat[(a)]{\includegraphics[width=160pt]{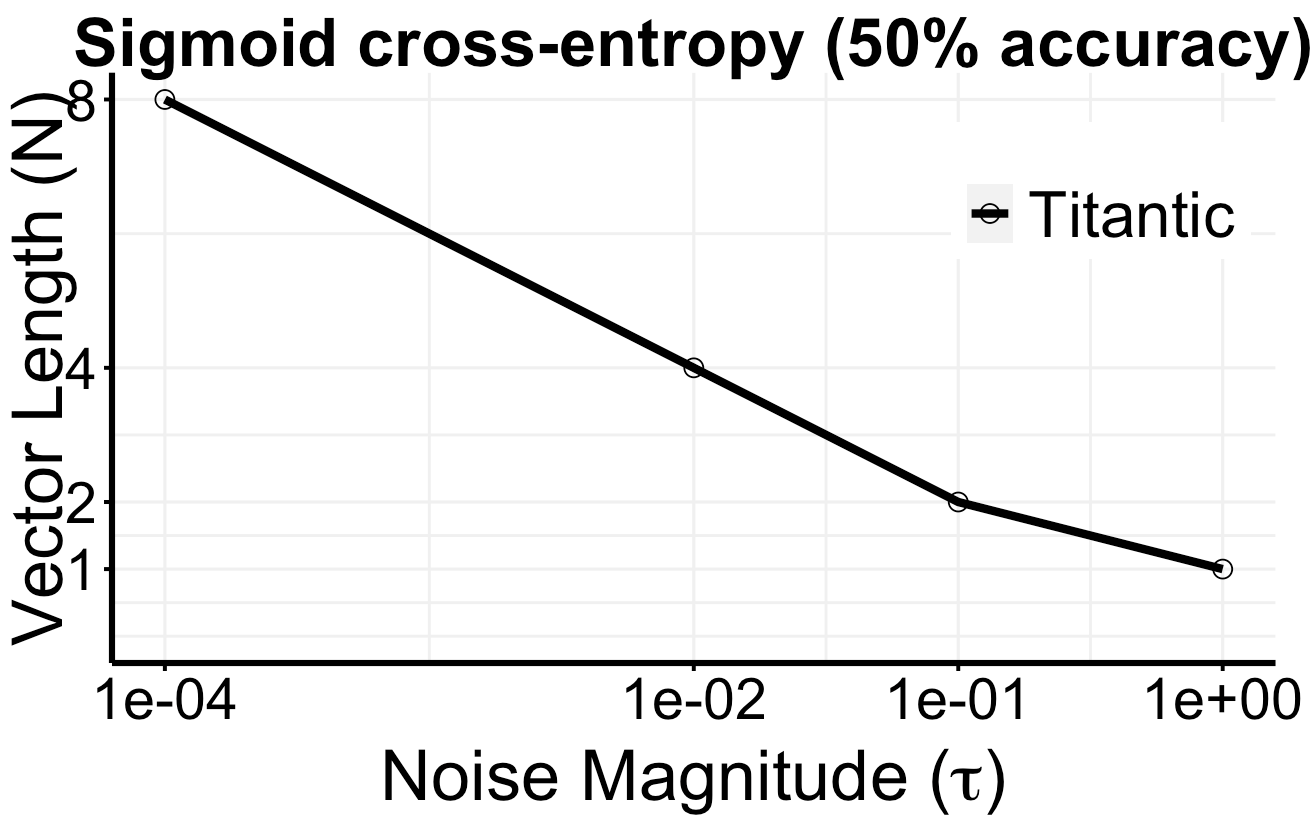}} 
    \hspace{2em}
    \subfloat[(b)]{\includegraphics[width=160pt]{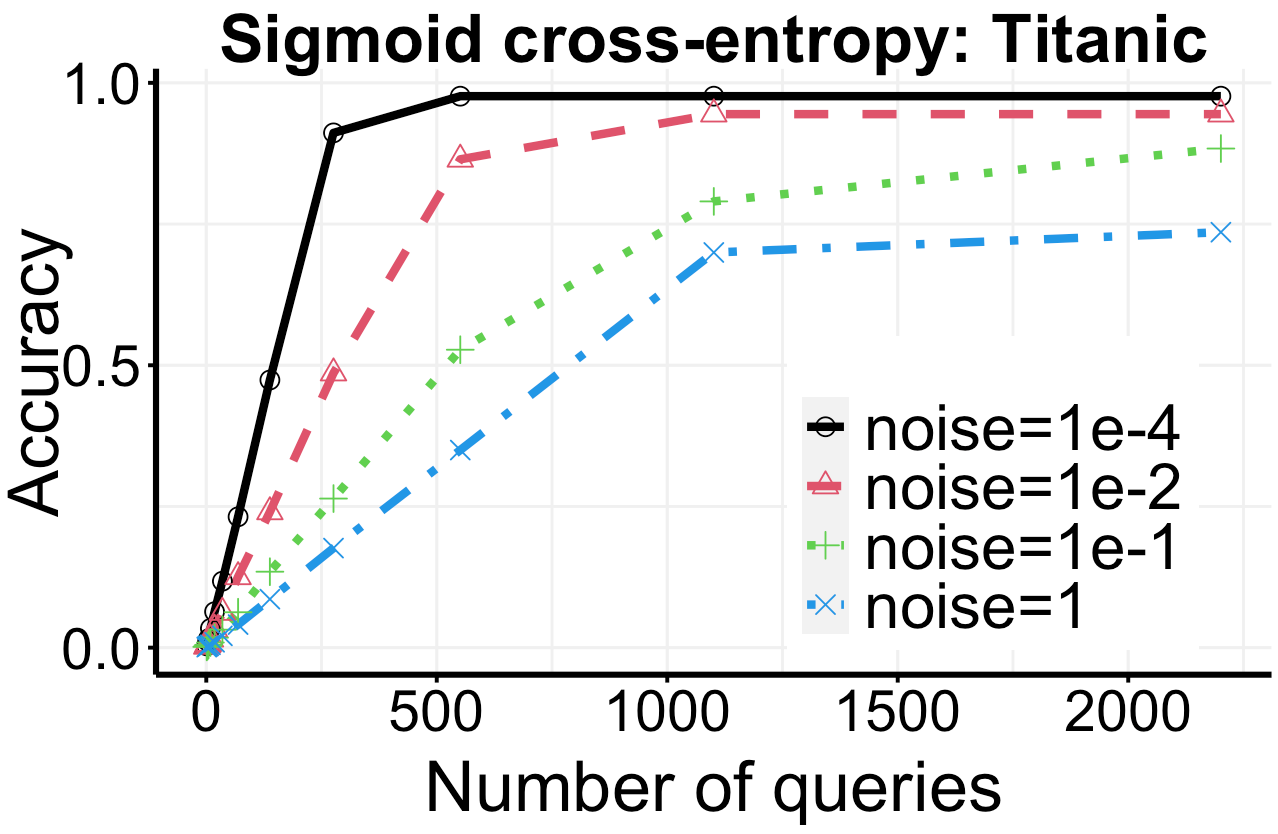}}
    \vspace{-1ex}
    \caption{The plot on the left shows the length of vector recovered (at 50\% accuracy) using single query. The plot on the right shows label reconstruction accuracy with the multi-query label inference attack.}
    \label{fig:sigmoid_experiments}
       \vspace{1ex}
\end{figure}

We now discuss the missing details from Section~\ref{sec:experiments} and present our results for label inference from the Sigmoid cross entropy loss function.


\noindent\textbf{Label Inference from Binary Cross-Entropy Loss.} In our experiments, for binary cross-entropy, we use the label inference attack of \citep{aggarwal2021icml} as a baseline (see Figures~\ref{fig:singlequery}(a) and (b)). 

\begin{theorem}[$\tau$-codomain Separability from Algorithm 2 in~\citep{aggarwal2021icml}]\label{thm:from_icml}
Let $\tau > 0$. For the binary case (with class labels $0$ and $1$), define $\theta_i = \parfrac{3^{2^i}N\tau}{1+3^{2^i}N\tau}$ for all $i \in [N]$. Then, $\logloss$ is $2 \tau$-codomain separable using $\theta$.
\end{theorem}

Next, we discuss some technical caveats about the results for the softmax cross entropy loss, as observed in Figure~\ref{fig:singlequery}(d).  

\noindent\textbf{Additional bits Needed for Softmax Cross-Entropy Loss.} Recall from our discussion in Section~\ref{sec:experiments} that computing the softmax cross-entropy loss will require an additional $\Omega(NK + \ln(N\tau))$ bits over those required for the multiclass cross-entropy loss. We now formally argue this result.

Observe that for label inference in the softmax case, it suffices to compute a vector $\theta' = [\theta'_1,\dots,\theta'_N]$ such that $\textsc{Softmax}(\theta'_i) = \theta$, where $\theta$ is our desired vector for label inference. This is equivalent to requiring: $e^{\theta'_i}/\sum_j e^{\theta'_j} = \theta_i$, which gives rise to:
\[\frac{e^{x_1}}{\theta_1} = \cdots = \frac{e^{x_N}}{\theta_N}.\] 
Thus, for any $i$ and $j$, we can write $\theta'_i = \theta'_j + \ln \parfrac{\theta_i}{\theta_j}$. Now, let 
$$i_{\shortuparrow} = \arg\max_{i \in [N]}\theta_i \mbox{ and } i_{\shortdownarrow} = \arg\min_{i \in [N]}\theta_i.$$ 
Then, we can write $x_{i_{\shortuparrow}} = x_{i_{\shortdownarrow}} + \ln \parfrac{\theta_{i_{\shortuparrow}}}{\theta_{i_{\shortdownarrow}}}$. Thus, the additional number of bits required to represent the entries in $x$ is $\Omega\paran{\ln \ln \parfrac{\theta_{i_{\shortuparrow}}}{\theta_{i_{\shortdownarrow}}} - \ln \theta_{i_{\shortuparrow}}} = \Omega\paran{\ln \ln \parfrac{\theta_{i_{\shortuparrow}}}{\theta_{i_{\shortdownarrow}}}}$. From our construction in Theorem~\ref{thm:multi_class_proof}, we know that the ratio 
$\frac{\theta_{i_{\shortuparrow}}}{\theta_{i_{\shortdownarrow}}} \approx 3^{2^{NK}N\tau}$. This means that we need an additional $\Omega\paran{\ln \ln 3^{2^{NK}N\tau}} = \Omega\paran{NK + \ln (N\tau)}$ bits for the softmax cross-entropy loss as compared to the (plain) multiclass cross-entropy loss.

\textbf{Additional Experimental Results.} Figure \ref{fig:logloss} presents the results with multiclass cross-entropy loss on Titanic, Satellite, MNIST, and CIFAR datasets, using the same setting as in the multi-query experiments in Figure~\ref{fig:multiquery}.


 Figure~\ref{fig:sigmoid_experiments} presents the results on sigmoid cross-entropy which is applicable only in the binary labeled setting (see~\eqref{eq:sigmoid_definition}) on the Titanic dataset. 
The results are worse compared to (plain) multiclass cross-entropy loss.
This is because the number of bits required to represent $\textsc{Sigmoid}^{-1}(\theta_i)$ (where $\theta_i$ is an element of $\theta$) is $\Omega\paran{\ln\abs{\ln (\theta_i/(1-\theta_i))}}$, which asymptotically dominates the $\Omega\paran{\abs{\ln \theta_i}}$ many bits required to represent $\theta_i$ for the (plain) multiclass cross-entropy loss.